\documentclass[pdflatex,sn-mathphys-num]{sn-jnl}

\usepackage{graphicx}%
\usepackage{tikz}%
\usetikzlibrary{arrows.meta,positioning}%
\usepackage{multirow}%
\usepackage{amsmath,amssymb,amsfonts}%
\usepackage{amsthm}%
\usepackage{mathrsfs}%
\usepackage{bm}%
\usepackage[title]{appendix}%
\usepackage{xcolor}%
\usepackage{textcomp}%
\usepackage{manyfoot}%
\usepackage{booktabs}%
\usepackage{algorithm}%
\usepackage{algorithmicx}%
\usepackage{algpseudocode}%
\usepackage{listings}%
\usepackage{placeins}%

\theoremstyle{thmstyleone}%
\newtheorem{theorem}{Theorem}%
\newtheorem{proposition}[theorem]{Proposition}%
\newtheorem{corollary}[theorem]{Corollary}%
\theoremstyle{thmstyletwo}%
\newtheorem{remark}{Remark}%
\theoremstyle{thmstylethree}%
\newcommand{\dd}{\,\mathrm{d}}
\newcommand{\R}{\mathbb{R}}
\newcommand{\inner}[2]{\langle #1,\, #2\rangle}
\newcommand{\Xop}{\mathcal{X}}
\newcommand{\Top}{\mathcal{T}}
\newcommand{\Lop}{\mathcal{L}}
\newcommand{\Rop}{\mathcal{R}}
\newcommand{\Tmat}{\Theta}
\newcommand{\eps}{\varepsilon}
\DeclareMathOperator*{\argmin}{arg\,min}

\begin{document}

\title[Weak-form fractional differential equation discovery]
{Robust data-driven discovery of fractional differential equations via weak formulations and Pareto-based subset selection}

\author*[1]{\fnm{Pongpisit} \sur{Thanasutives}}\email{pongpisit.thanasutives@riken.jp}

\author[1,2]{\fnm{Yoshinobu} \sur{Kawahara}}

\affil*[1]{\orgdiv{Center for Advanced Intelligence Project (AIP)}, \orgname{RIKEN}, \orgaddress{\city{Tokyo}, \country{Japan}}}

\affil[2]{\orgdiv{Graduate School of Information Science and Technology}, \orgname{The University of Osaka}, \orgaddress{\city{Osaka}, \country{Japan}}}

\abstract{
Fractional partial differential equations describe nonlocal dynamics, but discovering them from noisy data is difficult because fractional differentiation amplifies high-frequency measurement noise and the derivative orders are unknown. We propose Weak-Pareto, which combines an adjoint-consistent weak formulation of fractional terms with Pareto-based subset selection over discrete term types and continuous fractional orders. For linear right-hand-side terms, the adjoint transfers fractional operators from measured fields to smooth test functions, replacing noise-sensitive pointwise differentiation with smoothing integration; for nonlinear terms, the noise-suppression effect is partial yet useful. Coefficients are fitted by ridge regression within a branch-aware differential-evolution search over the orders. The support size is then selected at the validation-error--complexity elbow. We show that the variance of fixed linear right-hand-side weak features vanishes under grid refinement, whereas noise amplification in pointwise fractional features increases with derivative order. Across fractional advection--diffusion, reaction--diffusion, and Burgers benchmarks, Weak-Pareto recovers parsimonious structures from clean and noisy measurements. In controlled advection--diffusion and Burgers comparisons, it retains the correct support at every tested multiplicative-noise level, whereas the unregularised strong-form counterpart largely fails once noise is introduced; this advantage persists under additive Gaussian noise. Ablations show that the weak library drives noise robustness and that continuous-order Pareto search avoids the support-selection failure of a dense fixed dictionary. On the advection--diffusion benchmark, Weak-Pareto yields more consistent operator recovery and substantially lower measured runtime than a neural baseline. A two-dimensional extension of Weak-Pareto can recover distinct coordinate-dependent orders, including a three-term anisotropic model under noisy conditions.
}



\maketitle

\section*{Graphical Abstract}
\begin{center}
\includegraphics[width=\textwidth]{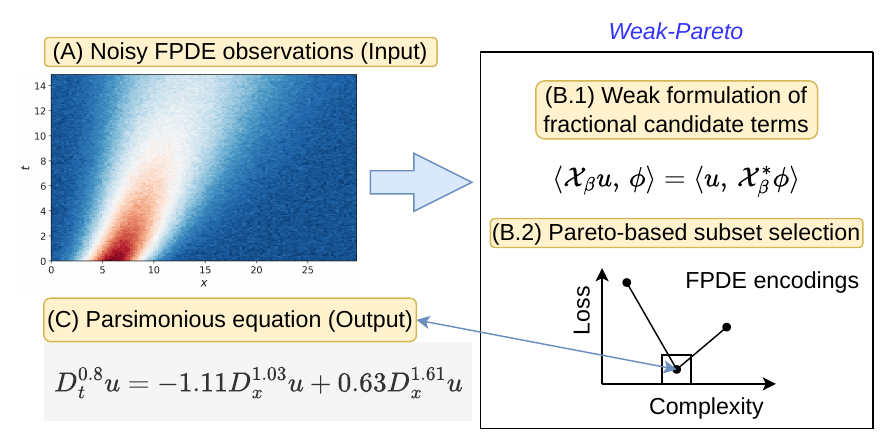}
\end{center}

\section{Introduction}\label{sec:intro}

Data-driven equation discovery aims to infer an interpretable, closed-form model directly from observations, rather than deriving it purely from first principles. Methods such as sparse identification of nonlinear dynamics (SINDy) and PDE functional identification (PDE-FIND) select a small set of active terms from an overcomplete candidate library by combining regression with sparsity-promoting model selection \cite{brunton2016sindy,rudy2017pdefind,schaeffer2017}. Their central principle is parsimony: the governing equation should contain only the important terms needed to explain the observed dynamics \cite{mangan2017modelselection}.

Many important real-world systems are nonlocal. Fractional differential equations, including fractional partial differential equations (FPDEs), describe anomalous diffusion, viscoelasticity, transport in heterogeneous media, and several biological and financial processes through real-valued derivative orders and nonlocal memory or spatial action \cite{podlubny1999,metzler2000anomalous,kilbas2006}. Discovering these equations would allow the data to reveal both the governing terms and the degree of nonlocality.

However, fractional equation discovery introduces two difficulties beyond the integer-order setting. First, fractional differentiation can amplify high-frequency measurement noise. For the periodic spectral operators used in several benchmarks, this mechanism is explicit in the frequency domain: an operator of order $\beta$ scales a Fourier mode of wavenumber $\kappa$ by a factor whose magnitude grows as $|\kappa|^\beta$. Strong-form methods therefore become increasingly fragile as the order or noise level grows. Second, the unknown orders are continuous. A fixed dictionary must either use a coarse order grid, which creates discretisation bias, or a dense grid, which produces many nearly collinear columns and thus destabilises support selection.

Weak formulations address the first difficulty for integer-order equations. Multiplying by smooth test functions and integrating by parts transfers derivatives from the measurements to the test functions, replacing pointwise derivative estimates with integral measurements. This principle underlies weak SINDy and related integral, Galerkin, and neural weak-form methods, which are substantially more noise-robust than strong-form regression \cite{gurevich2019,reinbold2020,messenger2021wsindy_pde,messenger2021wsindy_mms,schaeffer2017integral,tang2023weakident,fasel2022esindy,stephany2024weakpdelearn}. Complementary uncertainty-aware model-selection methods, including the uncertainty-penalised Bayesian information criterion (UBIC), improve robustness by penalising candidate terms whose inferred coefficients exhibit high uncertainty \cite{thanasutives2024ubic}.

Existing fractional-discovery methods, however, still evaluate fractional derivatives pointwise, either directly on a mesh or after reconstructing the field with a neural network \cite{gulian2019,pang2019fpinn,yu2025fde}. Extending the weak formulation to fractional operators is not trivial, because Caputo, Riemann--Liouville, Gr\"unwald--Letnikov, Riesz, and periodic spectral derivatives have different adjoints and boundary terms. A weak feature is valid only when it uses the adjoint of the operator being identified. To address the aforementioned difficulties, we propose \emph{Weak-Pareto}, a data-driven discovery framework that couples an adjoint-consistent weak formulation of the fractional candidate library with continuous-order, Pareto-based subset selection. Table~\ref{tab:method-positioning} characterises the differences between Weak-Pareto and existing equation-discovery methods. Note that notable symbolic-regression discovery frameworks \cite{schmidt2009distilling,cranmer2020discovering} search over algebraic expression trees built from a fixed operator set, and usually do not support a continuously varying differential order.

\begin{table}[t]
\centering
\footnotesize
\setlength{\tabcolsep}{0.2pt}
\caption{Comparison with prior equation-discovery methods. ``Cont.\ order'' means real-valued order optimisation rather than selection from a fixed order dictionary. ``Search strategy'' describes how active terms or a prescribed model are identified; ``Frac.\ adjoint'' means that weak features use the adjoint and boundary terms of the represented fractional operator.}\label{tab:method-positioning}
\begin{tabular*}{\textwidth}{@{\extracolsep\fill}lccccc}
\toprule
Method & Weak & Fractional & Cont.\ order & Search strategy & Frac.\ adjoint \\
\midrule
WSINDy / WeakIdent \cite{messenger2021wsindy_pde,tang2023weakident} & yes & no & no & sparse regression & no \\
Weak-PDE-LEARN \cite{stephany2024weakpdelearn} & yes & no & no & sparse regression & no \\
Gulian et al.\ \cite{gulian2019} & no & yes & yes & prescribed structure & no \\
Yu et al.\ \cite{yu2025fde} & no & yes & yes & sparse regression & no \\
Weak-Pareto (this work) & yes & yes & yes & best subset & yes \\
\botrule
\end{tabular*}
\end{table}

To our knowledge, Weak-Pareto is the first equation-discovery framework to achieve all five properties in Table~\ref{tab:method-positioning}. Our core contributions are threefold.

\begin{enumerate}
\item \textbf{An adjoint-consistent weak library for fractional operators} (Section~\ref{sec:weak}). Each candidate uses the adjoint and boundary terms of its declared operator. For linear right-hand-side terms, the fractional derivative acts only on smooth test functions; those library columns therefore do not differentiate the measured field. Proposition~\ref{prop:variance} shows that the variance of a fixed linear right-hand-side weak feature vanishes under grid refinement, whereas the variance of a pointwise positive-order feature diverges at a rate that increases with derivative order. For nonlinear terms, we show explicitly that the weak row is an averaged projection of the corresponding strong feature and quantify the resulting noise-induced bias. For a fixed support, Proposition~\ref{prop:local-identifiability} further characterises when a spatial-order perturbation can be absorbed, to first order, by refitting the active coefficients.
\item \textbf{A branch-aware, continuous-order Pareto search} (Section~\ref{sec:pareto}). Candidate terms are encoded as integer powers paired with continuous fractional orders. Coefficients are fitted analytically inside a differential-evolution search over the orders, while subunit, exact-integer, and superunit Caputo modes are treated as distinct branches. The support size is then selected at the validation-error--complexity elbow. This avoids both the discretisation error imposed by coarse fixed dictionaries and the severe collinearity of dense ones. After selecting the support and temporal branch, we evaluate the operators at locally refined orders and refit the coefficients on the full data. This removes interpolation error from the reported coefficients, although the selected model can still depend on the original order grid and search trajectory.
\item \textbf{Controlled analyses and baseline comparisons} (Section~\ref{sec:exp}). Library-only ablations isolate the advantage of weak over pointwise measurements; fixed-dictionary ablations isolate the advantage of continuous-order search; and a controlled comparison assesses an adapted neural fractional-discovery framework on the advection--diffusion benchmark. A semi-analytic fixed-support diagnostic tests the previously unreported superunit temporal branch. The experiments also identify the method's present limits on Riesz reaction--diffusion problems, demonstrate weak-form integral fitting on irregular frozen-soil creep data, and use an anisotropic two-dimensional example to showcase direction-labelled continuous-order encoding.
\end{enumerate}

On the fractional advection--diffusion (FADE) and fractional Burgers benchmarks, Weak-Pareto recovers the correct support in all five seeds at every tested multiplicative-noise level up to $20\%$. Under matched selection, the strong-form framework achieves no complete operator recovery in any noisy run. The two-dimensional example recovers the support, coordinate directions, and derivative orders in all 25 runs for each of its two benchmarks.

Fig.~\ref{fig:overview} connects the main stages of Weak-Pareto to concrete results from the supplied experiments. Panel~(a) visualises the FADE benchmark under $10\%$ additive Gaussian noise. Panel~(b) combines the adjoint transfer used for linear weak features with the branch-aware Pareto search over support size, discrete term types, and continuous fractional orders. Panel~(c) displays the closed-form encoding, the correct-support recovery pattern in the additive-Gaussian comparisons, the representative FADE discovery reported in Appendix~\ref{app:equations}, and an example noisy estimate for the three-term anisotropic two-dimensional benchmark~\eqref{eq:twod-b}. The recovery counts refer specifically to correct-support recovery; complete operator recovery additionally requires the fractional orders to satisfy the declared tolerances.

\begin{figure}[t]
\centering
\includegraphics[width=\textwidth]{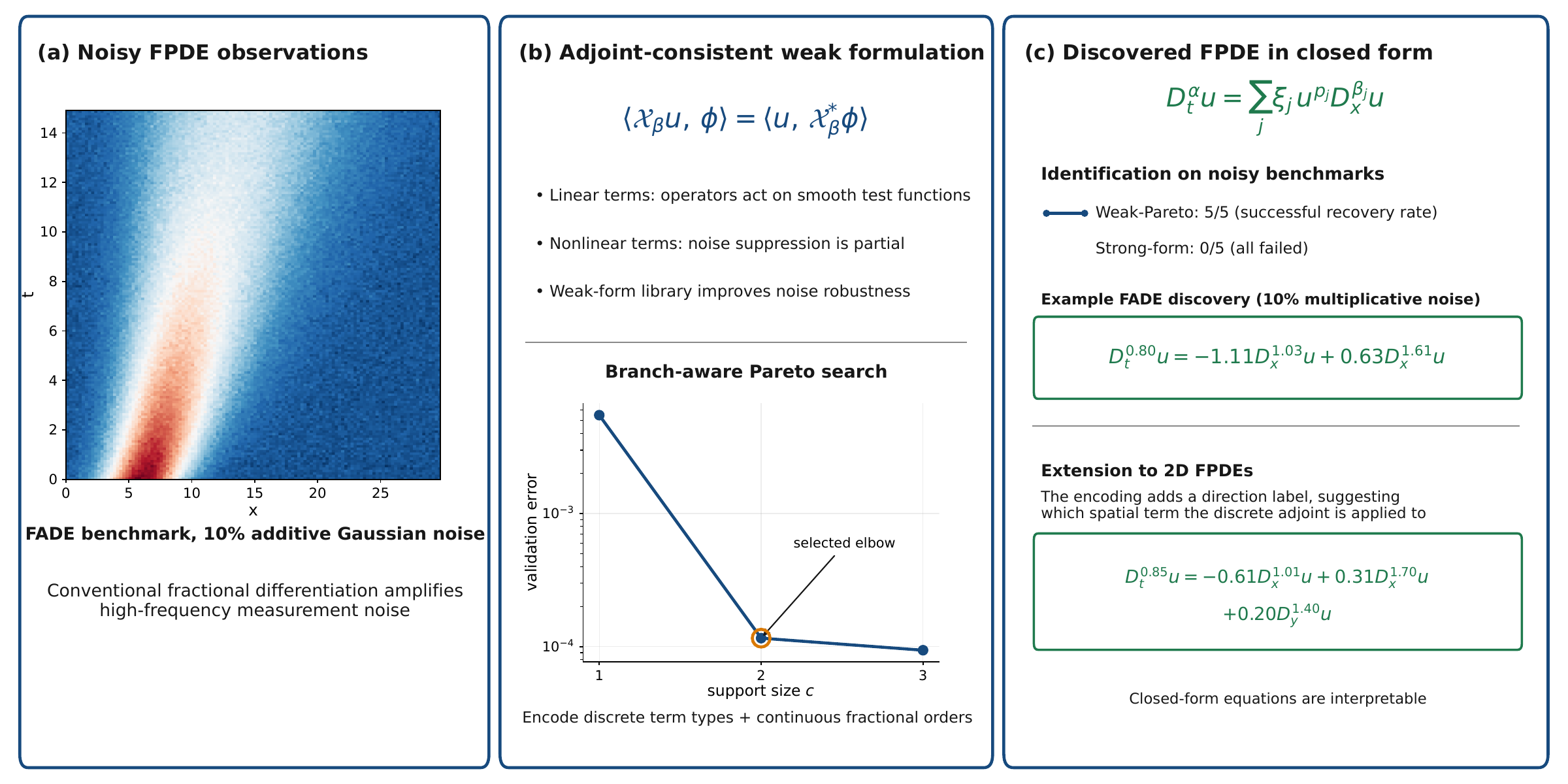}
\caption{Overview of Weak-Pareto. (a) Visualisation of the FADE benchmark with $10\%$ additive Gaussian noise, illustrating the high-frequency perturbations amplified by fractional differentiation. (b) For a linear operator $\mathcal X_\beta$, the adjoint identity transfers differentiation from the measured field to a smooth test function. The lower schematic shows the branch-aware Pareto search over support size, discrete term types, and continuous fractional orders, with the selected validation-error--complexity elbow highlighted. Nonlinear weak features remain averaged measurements; therefore, their noise suppression is partial. (c) The candidate encoding and its discoveries remain explicit closed-form equations. For both the FADE and fractional Burgers experiments at $10\%$ additive Gaussian noise, Weak-Pareto recovers the correct support in $5/5$ runs, whereas the strong-form counterpart fails in all five runs (Appendix~\ref{app:altnoise}). By augmenting the encoding tuple with a direction label, we extend Weak-Pareto to facilitate data-driven FPDE discovery in two-dimensional settings~\eqref{eq:twod-b}}
\label{fig:overview}
\end{figure}

The remainder of the paper defines the model class (Section~\ref{sec:problem}), develops the weak library and its noise analysis (Section~\ref{sec:weak}), presents the Pareto search (Section~\ref{sec:pareto}), reports the experiments (Section~\ref{sec:exp}), and concludes in Section~\ref{sec:conclusion}. Proofs, numerical verification, and detailed settings are provided in the appendices.

\section{Problem formulation}\label{sec:problem}

\subsection{Fractional derivative operators}\label{subsec:operators}

We first consider a scalar field $u(t,x)$ on $Q=(0,T)\times\Omega$, with $\Omega\subset\R$, observed on a space--time grid. This notation covers the core method and analysis; Section~\ref{subsec:twod} gives an example extension to two spatial coordinates. The temporal operator is denoted by $\Top_{m_\alpha,\alpha}$, where $m_\alpha$ identifies the Caputo branch and $\alpha$ its order. We distinguish the subunit branch $m_\alpha=\mathrm{sub}$ for $0<\alpha<1$, the exact integer mode $m_\alpha=\mathrm{int}$ at $\alpha=1$, and, when the declared range extends above one, the superunit branch $m_\alpha=\mathrm{sup}$ for $1<\alpha<2$. Spatial terms use operators $\Xop_\beta$ of order $\beta$. We adopt the following standard definitions \cite{podlubny1999,kilbas2006}. For $\mu>0$ and noninteger $\gamma>0$, let $n=\lceil\gamma\rceil$. The left and right Riemann--Liouville (RL) integrals and derivatives on $[a,b]$ are
\begin{equation}
\begin{aligned}
({}_aI_z^{\mu}f)(z)&=\frac{1}{\Gamma(\mu)}\int_a^z (z-s)^{\mu-1}f(s)\dd s,
&({}_zI_b^{\mu}f)(z)&=\frac{1}{\Gamma(\mu)}\int_z^b (s-z)^{\mu-1}f(s)\dd s,\\
({}_aD_z^\gamma f)(z)&=\frac{\dd^n}{\dd z^n}\,({}_aI_z^{n-\gamma}f)(z),
&({}_zD_b^\gamma f)(z)&=(-1)^n\frac{\dd^n}{\dd z^n}\,({}_zI_b^{n-\gamma}f)(z).
\end{aligned}
\label{eq:rl-def}
\end{equation}
Here $\Gamma(\mu)=\int_0^\infty y^{\mu-1}e^{-y}\dd y$ is Euler's gamma function; integer-order derivatives have their usual meaning. The Caputo derivative subtracts the initial Taylor polynomial, ${}_a^{C}\!D_z^\alpha f={}_aD_z^\alpha[f-P_{n-1,a}f]$ with $P_{n-1,a}f(z)=\sum_{q=0}^{n-1}\frac{f^{(q)}(a)}{q!}(z-a)^q$ and $n=\lceil\alpha\rceil$. Below, $D_t^\alpha$ abbreviates ${}_0^{C}\!D_t^\alpha$ for noninteger $\alpha$, while $\partial_t$ denotes the exact integer operator. The right RL derivatives in Eq.~\eqref{eq:rl-def} arise as the adjoints of the corresponding left-sided operators (Section~\ref{subsec:adjoints}). On periodic domains we also use the Riesz operator and the directional spectral derivative, defined through their Fourier multipliers
\begin{equation}
\mathcal F[\Rop_\beta f](\kappa)=-|\kappa|^{\beta}\,\mathcal F[f](\kappa),
\qquad
\mathcal F[D_x^\beta f](\kappa)=(\mathrm i\kappa)^{\beta}\,\mathcal F[f](\kappa).
\label{eq:spectral-def}
\end{equation}
Here $\mathcal F$ denotes the Fourier transform and $\kappa$ is the spatial wavenumber. The Riesz operator $\Rop_\beta=-(-\Delta)^{\beta/2}$ has the non-positive multiplier $-|\kappa|^\beta$; thus $\beta=2$ recovers $\partial_x^2$, and a positive coefficient produces diffusion. The directional multiplier $(\mathrm i\kappa)^\beta$ uses the principal branch,
\begin{equation}
(\mathrm i\kappa)^{\beta}:=|\kappa|^{\beta}\exp\!\Bigl(\mathrm i\tfrac{\pi}{2}\beta\,\operatorname{sgn}(\kappa)\Bigr),\qquad \kappa\neq 0,
\label{eq:ik-branch}
\end{equation}
where $\operatorname{sgn}$ is the sign function and the zero mode is annihilated. The multiplier is conjugate-symmetric, $\overline{(\mathrm i\kappa)^\beta}=(\mathrm i(-\kappa))^\beta$; consequently, it maps real fields to real fields and models advective fractional transport; $\beta=1$ recovers $\partial_x$.

The searched model class is the parsimonious fractional equation
\begin{equation}
\Top_{m_\alpha,\alpha}u=\sum_{j=1}^{c}\xi_j\,u^{p_j}\,\Xop_{\beta_j}u.
\label{eq:model}
\end{equation}
The model has $c$ active terms, with integer powers $p_j$, spatial orders $\beta_j$, and coefficients $\xi_j$. The class includes fractional advection--diffusion ($p_j=0$) and nonlinear transport such as the Burgers term $u\,\partial_xu$ ($p=1,\beta=1$). We define $\Xop_0$ as the identity; hence $(p,0)$ represents the reaction term $u^{p+1}$. This is a modelling convention, not the $\beta\to0$ limit of the Riesz multiplier, which is $-(\mathrm{Id}-\Pi_0)$ because the zero Fourier mode is annihilated; it reduces to $-\mathrm{Id}$ only on mean-zero fields, and any sign is absorbed into $\xi_j$.

\subsection{FPDE encoding}\label{subsec:encoding}

Each candidate is represented by the tuple
\begin{equation}
\mathcal M=(m_\alpha,\alpha,\bm p,\bm\beta,\bm\xi),\qquad
\bm p=(p_1,\dots,p_c),\quad \bm\beta=(\beta_1,\dots,\beta_c),\quad \bm\xi=(\xi_1,\dots,\xi_c).
\end{equation}
Here $m_\alpha$ records the temporal branch. This label is crucial at an integer: an estimate such as $\alpha=0.999$ remains a subunit Caputo model and is not identified with $\partial_t$. The subunit, exact-integer, and superunit branch minima are therefore optimised separately and compared using the same validation objective.

Each right-hand-side term is represented by a discrete power $p_j$ and a continuous order $\beta_j$. Unlike fixed-library regression, this encoding does not enumerate all possible orders in advance. A finite fixed grid either omits the true order or becomes increasingly collinear as it is refined (Section~\ref{subsec:ablation}); Weak-Pareto instead assembles only the columns proposed by the search. For example, the FADE equation $D_t^{0.8}u=-\partial_xu+0.5D_x^{1.7}u$ is encoded by $m_\alpha=\mathrm{sub}$, $\alpha=0.8$, $\bm p=(0,0)$, $\bm\beta=(1,1.7)$, and $\bm\xi=(-1,0.5)$. At each support size $c$, the method searches for the best $c$-term equation. Within a fractional temporal branch it optimises $(\alpha,\beta_1,\dots,\beta_c)$; in the exact integer mode, $\alpha=1$ is fixed and only $\bm\beta$ is optimised. For fixed modes, powers, and orders, the coefficients follow from linear regression (Section~\ref{sec:pareto}). In multiple spatial coordinates, our encoding can be augmented by a direction label $d_j$ for each term. The example in Section~\ref{subsec:twod} uses $d_j\in\{x,y\}$ and searches the direction pattern together with the continuous orders; its reported implementation is restricted to linear terms.

\section{Weak formulation of fractional candidate terms}\label{sec:weak}

This section develops the first core contribution: weak features that are consistent with the adjoint and boundary structure of each fractional operator, together with a feature-level explanation of their noise-robustness advantage over pointwise differentiation.

\subsection{Test functions and weak features}\label{subsec:weakresidual}

Let $\{\phi_k\}_{k=1}^{K}$ be smooth, separable test functions
\begin{equation}
\phi_k(t,x)=\vartheta_{\ell_t}(t)\,\psi_{\ell_x}(x),\qquad \phi_k\in C^{s}(\overline Q).
\end{equation}
The index $k$ enumerates temporal--spatial test-function pairs $(\ell_t,\ell_x)$. For a localised test family, we use the term \emph{test window} for a test function together with the region on which it has appreciable weight. Here $s\in\mathbb N_0$, $\overline Q=[0,T]\times\overline\Omega$, and $C^s(\overline Q)$ denotes functions whose partial derivatives of total order at most $s$ extend continuously to $\overline Q$; the smoothness order is chosen for the highest derivative in the library. The implementation provides compactly supported bumps, localised Gaussian test windows, and global Fourier modes. Reported experiments use Gaussian test windows and, for periodic high-order Riesz cases, Fourier spatial modes; bumps are not used in the reported results. In this subsection, ``compact support'' refers only to localisation of a test function. Elsewhere, the support of a candidate model means its active-term set in the subset search. Compact support of a test function or periodicity removes the corresponding continuum boundary terms. Gaussian tests use the exact discrete adjoint of Section~\ref{subsec:discrete}, preserving the discrete identity without a zero-trace assumption. Each test function produces one scalar equation,
\begin{equation}
\inner{\Top_{m_\alpha,\alpha}u}{\phi_k}=\sum_{j=1}^{c}\xi_j\,\inner{u^{p_j}\Xop_{\beta_j}u}{\phi_k},
\qquad \inner{f}{g}=\int_Q f\,g\dd x\dd t.
\label{eq:weak-residual}
\end{equation}
The construction rests on the adjoint identity
\begin{equation}
\inner{\Lop f}{\phi_k}=\inner{f}{\Lop^\ast\phi_k}+B_{\Lop}(f,\phi_k),
\label{eq:adjoint-general}
\end{equation}
where $\Lop^\ast$ is the adjoint and $B_{\Lop}$ contains boundary and initial contributions. The same $\phi_k$ is used on both sides; the operator $\Lop$ determines its adjoint and boundary contribution. Compact support of a test function or periodicity removes continuum spatial boundary terms. For Gaussian tests, the exact discrete adjoint replaces a zero-trace assumption; the Caputo target retains its initial-condition correction. Linear spatial operators transfer completely from the data to the test function. For $u^p\Xop_\beta u$, the discrete adjoint gives an integrated projection of the strong nonlinear feature, averaging it without removing differentiation from the noisy data path.

\subsection{Adjoint identities}\label{subsec:adjoints}

\paragraph{Temporal (Caputo) target.} Let $n=\lceil\alpha\rceil$ and $P_{n-1,0}u(t,x)=\sum_{q=0}^{n-1}\partial_t^qu(0,x)t^q/q!$. For $0<\alpha<2$, $\alpha\ne1$, and test functions satisfying the terminal conditions required to remove the right-endpoint traces, fractional integration by parts gives
\begin{equation}
\inner{{}_0^{C}\!D_t^\alpha u}{\phi_k}
=\inner{\,u-P_{n-1,0}u\,}{({}_tD_T^\alpha)\phi_k}.
\label{eq:caputo-weak}
\end{equation}
Thus the subunit branch subtracts $u(0,\cdot)$, whereas the superunit branch also subtracts $t\,\partial_tu(0,\cdot)$. At $\alpha=1$, the trace-free continuum identity is $\int_Q u_t\phi_k=-\int_Q u\,\partial_t\phi_k$; reported Gaussian tests instead use the discrete transpose identity of Section~\ref{subsec:discrete}. The domain integral leaves $u$ undifferentiated, although the superunit branch still depends on $\partial_tu(0,\cdot)$. In the reported L1 experiments, $D_h^\alpha=\mathsf{L1}_{\alpha-1}D_1$ and $(D_h^\alpha)^\top=D_1^\top\mathsf{L1}_{\alpha-1}^\top$. Its first weights are endpoint-concentrated, implicitly treating the initial rate one-sidedly. Hence the target is weak in the interior but not derivative-free at the initial boundary. Remark~\ref{rem:temporal-target} gives the scaling and noise analysis. The optional Volterra target estimates the initial Taylor polynomial explicitly but is not used here.

\paragraph{Spatial terms.} For a linear spatial term the weak feature is
\begin{equation}
\theta_{k,\beta}=\inner{\Xop_\beta u}{\phi_k}=\inner{u}{\Xop_\beta^\ast\phi_k},
\label{eq:linear-spatial}
\end{equation}
where the adjoint depends on the operator definition. For a left Riemann--Liouville derivative in $x$ on $[a,b]$, suppressing the other variables, the continuum identity is
\begin{equation*}
\int_a^b ({}_aD_x^\beta u)(x)\,\phi(x)\dd x
=\int_a^b u(x)\,({}_xD_b^\beta\phi)(x)\dd x,
\end{equation*}
under the endpoint conditions stated in Appendix~\ref{app:adjoints}; hence $({}_aD_x^\beta)^\ast={}_xD_b^\beta$ on that domain. For the periodic Riesz operator, $\Rop_\beta^\ast=\Rop_\beta$ because its multiplier $-|\kappa|^\beta$ is real and even. For the directional periodic operator,
\begin{equation*}
\mathcal F[\Xop_\beta^\ast\phi_k](\kappa)
=\overline{(\mathrm i\kappa)^\beta}\,\mathcal F[\phi_k](\kappa)
=(\mathrm i(-\kappa))^\beta\,\mathcal F[\phi_k](\kappa).
\end{equation*}
Using an adjoint from a different operator family changes the model being identified and can bias the recovered order. Each benchmark therefore uses the adjoint of its declared candidate operator.

\paragraph{Nonlinear terms.} For the nonlinear term $u^{p}\Xop_\beta u$ in~\eqref{eq:model},
\begin{equation}
\inner{u^{p}\Xop_\beta u}{\phi_k}
=\inner{\Xop_\beta u}{u^{p}\phi_k}
=\inner{u}{\Xop_\beta^\ast(u^{p}\phi_k)}.
\label{eq:nonlinear-weak}
\end{equation}
Eq.~\eqref{eq:nonlinear-weak} is the consistent weak form of the nonlinear candidate. At the discrete level, let $A$ represent $\Xop_\beta$. Eq.~\eqref{eq:discrete-adjoint} gives the exact chain
\begin{equation*}
\inner{u}{A^{\ast,h}(u^p\phi_k)}_h
=\inner{Au}{u^p\phi_k}_h
=\inner{u^pAu}{\phi_k}_h.
\end{equation*}
Thus the discrete weak nonlinear feature on the left is exactly the test-function projection of the corresponding pointwise strong feature $u^pAu$. For nonlinear terms, differentiation still acts on the measured field, while integration provides averaging. The resulting feature may therefore be biased, and the strongest noise guarantee applies to linear terms.

\subsection{Numerical computation of adjoint identities}\label{subsec:discrete}

The numerical library uses the discrete adjoint of the operator that defines each candidate. Let $A\in\R^{n\times n}$ be the real-grid matrix representing a one-dimensional discrete operator, let $W=\operatorname{diag}(w_0,\ldots,w_{n-1})\in\R^{n\times n}$ be the positive quadrature-weight matrix, and let $f,\phi\in\R^n$, with $\inner{f}{g}_h=f^{\!\top}Wg$. The discrete adjoint is
\begin{equation}
A^{\ast,h}=W^{-1}A^{\!\top}W,\qquad \inner{Af}{\phi}_h=\inner{f}{A^{\ast,h}\phi}_h.
\label{eq:discrete-adjoint}
\end{equation}
When $A=A_{\beta,h}$ discretises $\Xop_\beta$, we write $\Xop_\beta^{\ast,h}\phi:=A_{\beta,h}^{\ast,h}\phi$. For the uniform quadrature used here, $W$ is a scalar multiple of the identity and $A^{\ast,h}=A^\top$. A left Gr\"unwald--Letnikov derivative therefore contributes $(G_L^\gamma)^\top\phi$; periodic operators use the conjugate Fourier multiplier; and the Caputo target uses the transpose of its L1 matrix. Because the Caputo family changes definition at integer orders, the subunit branch, exact integer operator, and superunit branch are precomputed and searched separately. Interpolation and local polishing remain within the selected branch. The exact $\alpha=1$ candidate is compared directly with the fractional-branch minima, rather than approximated by a nearby fractional order. The superunit implementation is included and verified numerically. Appendix~\ref{app:discrete} verifies the adjoints, branch separation, and discretisation accuracy.

\subsection{The noise-robust weak library}\label{subsec:noise}

Let $u^\star$ denote the noise-free field and let the measured field be $u=u^\star+\eta$, where $\eta$ is zero-mean measurement noise. For a linear candidate ($p=0$), define the noise-free weak feature by $\theta^\star_{k,\beta}:=\inner{u^\star}{\Xop_\beta^\ast\phi_k}$ and the measured feature by $\theta_{k,\beta}:=\inner{u}{\Xop_\beta^\ast\phi_k}$. Their difference is
\begin{equation}
\theta_{k,\beta}-\theta^\star_{k,\beta}=\inner{\eta}{\Xop_\beta^\ast\phi_k},
\label{eq:weak-perturbation}
\end{equation}
an integral projection of the noise, whereas the strong form applies the fractional operator to the noise pointwise. Here \emph{pointwise} means that $\Xop_\beta u$ is evaluated at individual grid nodes before regression, rather than first being integrated or projected against a test function. For the periodic spatial operators used below, write the Fourier action as
\begin{equation*}
\mathcal F_x[\Xop_\beta f](\kappa)=s_\beta(\kappa)\,\mathcal F_x[f](\kappa),\qquad
s_\beta(\kappa)=
\begin{cases}
-|\kappa|^\beta, & \Xop_\beta=\Rop_\beta,\\
(\mathrm i\kappa)^\beta, & \Xop_\beta=D_x^\beta,
\end{cases}
\end{equation*}
hence both positive-order multipliers satisfy $|s_\beta(\kappa)|=|\kappa|^\beta$. The weak--strong contrast can therefore be stated precisely for independent grid noise.

\paragraph{Theoretical properties.} The analysis below separates three properties. First, adjoint consistency ensures that every weak column represents the declared fractional operator and its boundary convention. Second, for linear right-hand-side columns, integration averages independent measurement noise: Proposition~\ref{prop:variance} shows variance decay under grid refinement, while the corresponding unregularised strong feature becomes increasingly noisy. These results characterise individual features; estimator-level behaviour is assessed empirically. Corollary~\ref{cor:multiplicative} extends the variance result to the principal multiplicative-noise model, while Remarks~\ref{rem:nonlinear-bias} and~\ref{rem:temporal-target} describe the remaining nonlinear bias and Caputo endpoint sensitivity.

\begin{proposition}\label{prop:variance}
Let $\eta_{ij}$, $i=0,\dots,n_t-1$ and $j=0,\dots,n_x-1$, be independent, zero-mean random variables with variance $\sigma^2$ on a sequence of uniform grids over the fixed spatiotemporal domain $Q$. By \emph{grid refinement} we mean $n_t,n_x\to\infty$ on this fixed domain, with $h_t,h_x\to0$ and $h_t=\Theta(n_t^{-1})$, $h_x=\Theta(n_x^{-1})$. Define
\[
\inner{f}{g}_h=h_th_x\sum_{i=0}^{n_t-1}\sum_{j=0}^{n_x-1}f_{ij}g_{ij},
\qquad
\lVert v\rVert_h^2=h_th_x\sum_{i=0}^{n_t-1}\sum_{j=0}^{n_x-1}v_{ij}^2.
\]
Thus $\lVert\cdot\rVert_h$ is the quadrature-weighted discrete $L^2(Q)$ norm on grid samples. Let $\lVert v\rVert_{L^2(Q)}^2=\int_Q|v|^2\,\dd x\dd t$. For part (i), hold the test function $\phi_k$ fixed on $Q$ as the grid is refined and assume that it belongs to the adjoint domain, so that $\omega^h_{k,\beta}=\Xop_\beta^{\ast,h}\phi_k$ satisfies $\lVert\omega^h_{k,\beta}\rVert_h\to\lVert\Xop_\beta^\ast\phi_k\rVert_{L^2(Q)}<\infty$.
\begin{enumerate}
\item[(i)] The weak-feature perturbation~\eqref{eq:weak-perturbation} satisfies
\begin{equation*}
\operatorname{Var}\bigl(\inner{\eta}{\omega^h_{k,\beta}}_h\bigr)=\sigma^2h_th_x\lVert\omega^h_{k,\beta}\rVert_h^2=O(h_th_x).
\end{equation*}
Hence, at fixed domain size, its variance is $O((n_tn_x)^{-1})$ as both grid dimensions are refined.
\item[(ii)] For a periodic pointwise feature whose multiplier $s_\beta$ is defined above and satisfies $|s_\beta(\kappa)|=|\kappa|^\beta$,
\begin{equation*}
\operatorname{Var}\bigl((\Xop_\beta\eta)_{ij}\bigr)
=\frac{\sigma^2}{n_x}\sum_{\ell=-\lfloor n_x/2\rfloor}^{\lceil n_x/2\rceil-1}|s_\beta(\kappa_\ell)|^2
\sim\frac{\pi^{2\beta}}{2\beta+1}\sigma^2h_x^{-2\beta},
\qquad \kappa_\ell=\frac{2\pi\ell}{L_x}.
\end{equation*}
Thus the variance of every unregularised positive-order strong-form feature ($\beta>0$) diverges under spatial refinement, at a rate that increases with $\beta$; the separately defined identity candidate at $\beta=0$ retains variance $\sigma^2$. For an even grid, a real-valued directional implementation may treat the single unpaired Nyquist mode separately; this changes one summand only and leaves the asymptotic relation unchanged, as noted in the proof.
\end{enumerate}
\end{proposition}

\begin{proof}
For (i), independence removes all cross-covariances:
\[
\operatorname{Var}\!\left(h_th_x\sum_{i=0}^{n_t-1}\sum_{j=0}^{n_x-1}\eta_{ij}(\omega^h_{k,\beta})_{ij}\right)
=\sigma^2(h_th_x)^2\sum_{i=0}^{n_t-1}\sum_{j=0}^{n_x-1}(\omega^h_{k,\beta})_{ij}^2
=\sigma^2h_th_x\lVert\omega^h_{k,\beta}\rVert_h^2.
\]
The assumed norm convergence makes the final factor bounded. It holds for the fixed smooth Fourier and periodised-Gaussian tests used with the periodic operators; for a one-sided operator, it holds whenever the fixed test function belongs to the corresponding adjoint domain.

For (ii), fix a time index and write $\eta_j$ for the spatial noise samples on that time slice. Let
\[
\mathcal L_{n_x}=\{-\lfloor n_x/2\rfloor,\ldots,\lceil n_x/2\rceil-1\},
\qquad \kappa_\ell=\frac{2\pi\ell}{L_x}.
\]
With the discrete Fourier transform
\[
\widehat\eta_\ell=\mathcal F_x[\eta](\kappa_\ell)
=\sum_{q=0}^{n_x-1}\eta_q e^{-\mathrm i\kappa_\ell x_q},
\]
the multiplier definition gives the inverse representation
\begin{equation*}
(\Xop_\beta\eta)(x_j)
=\frac{1}{n_x}\sum_{\ell\in\mathcal L_{n_x}}
 s_\beta(\kappa_\ell)\widehat\eta_\ell e^{\mathrm i\kappa_\ell x_j}.
\end{equation*}
Here $\delta_{qr}$ denotes the Kronecker delta. Since $\mathbb E[\eta_q\eta_r]=\sigma^2\delta_{qr}$, discrete Fourier orthogonality yields
\begin{align*}
\mathbb E[\widehat\eta_\ell\overline{\widehat\eta_m}]
&=\sum_{q=0}^{n_x-1}\sum_{r=0}^{n_x-1}
  \mathbb E[\eta_q\eta_r]
  e^{-\mathrm i\kappa_\ell x_q}e^{\mathrm i\kappa_m x_r}\\
&=\sigma^2\sum_{q=0}^{n_x-1}e^{-\mathrm i(\kappa_\ell-\kappa_m)x_q}
=n_x\sigma^2\delta_{\ell m}.
\end{align*}
The output is real for the conjugate-symmetric multipliers considered here. Its mean is zero, and expanding the squared magnitude therefore gives
\begin{align*}
\operatorname{Var}\bigl((\Xop_\beta\eta)(x_j)\bigr)
&=\frac{1}{n_x^2}
\sum_{\ell,m\in\mathcal L_{n_x}}
 s_\beta(\kappa_\ell)\overline{s_\beta(\kappa_m)}
 e^{\mathrm i(\kappa_\ell-\kappa_m)x_j}
 \mathbb E[\widehat\eta_\ell\overline{\widehat\eta_m}]\\
&=\frac{\sigma^2}{n_x}
\sum_{\ell=-\lfloor n_x/2\rfloor}^{\lceil n_x/2\rceil-1}
 |s_\beta(\kappa_\ell)|^2
=\frac{\sigma^2}{n_x}
\sum_{\ell=-\lfloor n_x/2\rfloor}^{\lceil n_x/2\rceil-1}
 |\kappa_\ell|^{2\beta}.
\end{align*}
Because $h_x=L_x/n_x$, the last sum obeys
\begin{align*}
\frac{h_x^{2\beta}}{n_x}
\sum_{\ell=-\lfloor n_x/2\rfloor}^{\lceil n_x/2\rceil-1}|\kappa_\ell|^{2\beta}
&=\frac1{n_x}
\sum_{\ell=-\lfloor n_x/2\rfloor}^{\lceil n_x/2\rceil-1}
\left|\frac{2\pi\ell}{n_x}\right|^{2\beta}\\
&\longrightarrow\int_{-1/2}^{1/2}|2\pi r|^{2\beta}\,\dd r
=\frac{\pi^{2\beta}}{2\beta+1}.
\end{align*}
Consequently,
\[
\operatorname{Var}\bigl((\Xop_\beta\eta)(x_j)\bigr)
\sim\frac{\pi^{2\beta}}{2\beta+1}\sigma^2h_x^{-2\beta},
\]
which is exactly the asymptotic relation stated in part~(ii). For the even-grid directional implementation, the real-valued projection affects only the unpaired Nyquist mode. Its contribution to the variance is $O(h_x^{-(2\beta-1)})$, one order lower than the $O(h_x^{-2\beta})$ total, and therefore does not change the leading constant or rate.
\end{proof}

Proposition~\ref{prop:variance} gives a feature-level explanation of why denser observations benefit the linear weak formulation more than the unregularised strong form under its assumptions. With the implementation's separate discrete $\ell^2$ row normalisation, the variance is $O(h_t^2h_x^2)$ for fixed separable tests, while signal and noise are rescaled identically. On a fixed domain, each linear weak measurement averages more independent noise samples and becomes more stable, whereas pointwise fractional differentiation admits progressively higher wavenumbers and amplifies their noise. Nonetheless, correlated noise, the Caputo initial-data terms, nonlinear features, and conditioning of the regression can alter the final estimator. We therefore assess coefficient and structure recovery empirically in Section~\ref{sec:exp}, and Appendix~\ref{app:discrete} verifies the predicted rates numerically.

\begin{corollary}[Multiplicative measurement noise]\label{cor:multiplicative}
Let the measured field be $u_{ij}=u^\star_{ij}(1+\rho\zeta_{ij})$, where the $\zeta_{ij}$ are independent, have zero mean and variance $\sigma_\zeta^2$, and $u^\star$ is regarded as fixed. For the same fixed test function on $Q$ and discrete weak-feature weights $\omega^h_{k,\beta}$ as in Proposition~\ref{prop:variance}(i), the conditional perturbation variance is
\begin{equation*}
\begin{aligned}
\operatorname{Var}\!\left(\inner{u-u^\star}{\omega^h_{k,\beta}}_h\,\middle|\,u^\star\right)
&=\rho^2\sigma_\zeta^2(h_th_x)^2
  \sum_{i=0}^{n_t-1}\sum_{j=0}^{n_x-1}\bigl(u^\star_{ij}(\omega^h_{k,\beta})_{ij}\bigr)^2 \\
&\le \rho^2\sigma_\zeta^2\lVert u^\star\rVert_\infty^2
  h_th_x\lVert\omega^h_{k,\beta}\rVert_h^2
=O(h_th_x).
\end{aligned}
\end{equation*}
For the uniform perturbations $\zeta_{ij}\sim\mathcal U[-1,1]$ used in the main experiments, $\sigma_\zeta^2=1/3$.
\end{corollary}

\begin{proof}
Condition on the noise-free field $u^\star$. The perturbation is
\[
\inner{u-u^\star}{\omega^h_{k,\beta}}_h
=\rho h_th_x\sum_{i=0}^{n_t-1}\sum_{j=0}^{n_x-1}u^\star_{ij}\zeta_{ij}(\omega^h_{k,\beta})_{ij}.
\]
Its conditional mean is zero. Independence again removes the cross-terms, giving
\[
\operatorname{Var}\!\left(\inner{u-u^\star}{\omega^h_{k,\beta}}_h\,\middle|\,u^\star\right)
=\rho^2\sigma_\zeta^2(h_th_x)^2\sum_{i=0}^{n_t-1}\sum_{j=0}^{n_x-1}(u^\star_{ij})^2(\omega^h_{k,\beta})_{ij}^2.
\]
Using $|u^\star_{ij}|\le\lVert u^\star\rVert_\infty$ and the definition of $\lVert\cdot\rVert_h$ gives, explicitly,
\begin{align*}
\operatorname{Var}\!\left(\inner{u-u^\star}{\omega^h_{k,\beta}}_h\,\middle|\,u^\star\right)
&\le \rho^2\sigma_\zeta^2\lVert u^\star\rVert_\infty^2
   (h_th_x)^2\sum_{i=0}^{n_t-1}\sum_{j=0}^{n_x-1}(\omega^h_{k,\beta})_{ij}^2\\
&=\rho^2\sigma_\zeta^2\lVert u^\star\rVert_\infty^2
   h_th_x\lVert\omega^h_{k,\beta}\rVert_h^2\\
&=O(h_th_x),
\end{align*}
where the final step uses the bounded-weight assumption from Proposition~\ref{prop:variance}(i).
\end{proof}

\begin{remark}[Bias of nonlinear weak features]\label{rem:nonlinear-bias}
For $p=1$, let $A_h=I_{n_t}\otimes A_x$ apply the discrete spatial operator $A_x$ independently at each time level, and let $D_{\phi_k}$ be the diagonal matrix formed from the vectorised test function. Under additive i.i.d.\ noise $u=u^\star+\eta$ with variance $\sigma^2$, the discrete-adjoint feature has the exact bias
\begin{equation*}
\mathbb E\!\left[\inner{u}{A^{\ast,h}(u\phi_k)}_h\right]
-\inner{u^\star}{A^{\ast,h}(u^\star\phi_k)}_h
=\sigma^2h_th_x\operatorname{tr}(A^{\ast,h}D_{\phi_k}).
\end{equation*}
Under the multiplicative model of Corollary~\ref{cor:multiplicative}, the corresponding conditional bias is
\begin{equation*}
\rho^2\sigma_\zeta^2h_th_x\operatorname{tr}\!\left(D_{u^\star}A^{\ast,h}D_{\phi_k}D_{u^\star}\right).
\end{equation*}
These formulas apply when powers are formed directly from the measured field. Some positive-valued datasets instead use $\max(u,0)^p$ to prevent a noisy observation from creating an unphysical negative base; this nonlinear clipping changes the expectation and hence the trace formula. Integration therefore provides averaging without, in general, eliminating nonlinear-feature bias. For the periodic first derivative, the diagonal of $A^{\ast,h}$ is zero and so is the leading additive-noise bias. For the periodic directional multiplier of order $\beta$, the spatial-adjoint diagonal satisfies
\begin{equation*}
(A_x^{\ast,h})_{jj}\sim \cos\!\left(\frac{\pi\beta}{2}\right)\frac{\pi^\beta}{\beta+1}\,h_x^{-\beta}.
\end{equation*}
Consequently, for a fixed test function with $\int_Q\phi_k\ne0$, the additive-noise bias has the leading behaviour
\begin{equation*}
\sigma^2h_th_x\operatorname{tr}(A^{\ast,h}D_{\phi_k})
\sim\sigma^2\cos\!\left(\frac{\pi\beta}{2}\right)\frac{\pi^\beta}{\beta+1}\,h_x^{-\beta}\int_Q\phi_k.
\end{equation*}
It is therefore negative for $1<\beta<2$ and grows as $O(h_x^{-\beta})$; if $\int_Q\phi_k=0$, this leading term vanishes. Appendix~\ref{app:discrete} verifies these statements numerically. For additive noise and unclipped quadratic features, the explicit trace formula also suggests a plug-in bias correction when the noise variance can be estimated; extending such corrections to multiplicative noise and clipped features is left for future work.
\end{remark}

\begin{remark}[Noise in the Caputo target]\label{rem:temporal-target}
Proposition~\ref{prop:variance} concerns linear right-hand-side columns, not the complete Caputo target in Eq.~\eqref{eq:caputo-weak}. Write one discrete target row as
\begin{equation*}
b_k=h_th_x\sum_{i=0}^{n_t-1}\sum_{j=0}^{n_x-1}\omega_{ij}u(t_i,x_j),
\end{equation*}
where $\omega_{ij}$ excludes the quadrature factor $h_th_x$. Here \emph{endpoint sensitivity} means that the target can depend much more strongly on noise in the first one or two time samples than on noise at a typical interior time. Two mechanisms below make this concrete. In the subunit branch, the same noisy observation $u(0,x_j)$ enters every temporal contribution through the Caputo initial-value subtraction, therefore adding more time levels does not create independent copies that can average out this noise. In the superunit branch, the transpose of the composed L1--first-difference operator places a large opposite-sign pair of weights on the first two time samples, making perturbations there comparatively influential.

\paragraph{Subunit branch.} The same noisy initial value $u(0,x_j)$ is reused in every temporal contribution. Its variance term is
\begin{equation*}
\sigma^2h_t^2h_x^2\sum_{j=0}^{n_x-1}
\left(\sum_{i=0}^{n_t-1}\omega_{ij}\right)^2.
\end{equation*}
For samples of a fixed continuum test weight, an ordinary full-field contribution has $\sum_{i,j}\omega_{ij}^2=O((h_th_x)^{-1})$, and multiplication by $h_t^2h_x^2$ gives the $O(h_th_x)$ variance in Proposition~\ref{prop:variance}. Reusing the initial sample instead gives
\begin{equation*}
\sum_{j=0}^{n_x-1}\left(\sum_{i=0}^{n_t-1}\omega_{ij}\right)^2
=O(h_t^{-2}h_x^{-1}),
\end{equation*}
and hence a variance of $O(h_x)$. The squared quadrature factors are therefore offset by the growth of the discrete sums. Spatial refinement still averages independent initial-time samples, but temporal refinement does not create new independent copies of the initial datum. This slower rate does not by itself determine the direction of an order-selection bias.

\paragraph{Superunit branch.} For $1<\alpha<2$, the reported evaluator uses $D_h^\alpha=\mathsf{L1}_{\alpha-1}D_1$. For the separable row $\phi_k(t,x)=\vartheta_{\ell_t}(t)\psi_{\ell_x}(x)$, consider unnormalised samples of the fixed continuum tests and define
\begin{equation*}
\bm w^{(\alpha)}=(D_h^\alpha)^\top\bm\vartheta_{\ell_t},
\qquad
\omega_{ij}=w_i^{(\alpha)}(\bm\psi_{\ell_x})_j,
\quad i=0,\ldots,n_t-1,\quad j=0,\ldots,n_x-1.
\end{equation*}
The factorisation also explains the endpoint weights. Set $\bm v=\mathsf{L1}_{\alpha-1}^{\top}\bm\vartheta_{\ell_t}$; then $\bm w^{(\alpha)}=D_1^\top\bm v$. For the fixed smooth temporal tests used under refinement, the L1 weights give $v_0=O(h_t^{-1})$, while the neighbouring entries $v_1$ and $v_2$ remain $O(1)$. Because the first row of $D_1$ is a forward difference and the next rows use centred differences, the first two transpose weights satisfy $w_0^{(\alpha)}=-v_0/h_t-v_1/(2h_t)$ and $w_1^{(\alpha)}=v_0/h_t-v_2/(2h_t)$. Hence $w_0^{(\alpha)},w_1^{(\alpha)}=O(h_t^{-2})$ and $w_0^{(\alpha)}=-w_1^{(\alpha)}+O(h_t^{-1})$. This opposite-sign initial pair therefore dominates $\sum_i|w_i^{(\alpha)}|^2=O(h_t^{-4})$. Moreover,
\begin{equation*}
h_x\sum_{j=0}^{n_x-1}|(\bm\psi_{\ell_x})_j|^2
\longrightarrow \lVert\psi_{\ell_x}\rVert_{L^2(\Omega)}^2,
\end{equation*}
hence $\sum_j|(\bm\psi_{\ell_x})_j|^2=O(h_x^{-1})$. For independent additive noise, the unnormalised target-row variance is therefore
\begin{equation*}
\sigma^2h_t^2h_x^2
\left(\sum_{i=0}^{n_t-1}|w_i^{(\alpha)}|^2\right)
\left(\sum_{j=0}^{n_x-1}|(\bm\psi_{\ell_x})_j|^2\right)
=O\!\left(\sigma^2h_xh_t^{-2}\right).
\end{equation*}
For the multiplicative model of Corollary~\ref{cor:multiplicative} with bounded $u^\star$, the same argument gives the upper bound $O(\rho^2\sigma_\zeta^2\lVert u^\star\rVert_\infty^2h_xh_t^{-2})$. Thus the endpoint sensitivity also applies to the noise law used in the superunit diagnostic. Under the implementation's separate discrete $\ell^2$ normalisation of the temporal and spatial test rows, the corresponding absolute variance is $O(\sigma^2h_x^2h_t^{-1})$. The clean target is rescaled by the same factors; this normalisation therefore does not remove the relative sensitivity to initial-time noise.

The practical implication is that the discrete transpose removes pointwise interior differentiation while retaining the initial-boundary dependence intrinsic to the Caputo derivative. Ordinary endpoint vanishing of $\vartheta_{\ell_t}$ leaves the initial pair coupled through $\mathsf{L1}_{\alpha-1}^\top$ and the first rows of $D_1^\top$. This behaviour is consistent with the initial-rate trace in the continuous identity. The complete target also contains the noisy full-field projection and is scored after variance normalisation in Eq.~\eqref{eq:val}.

Appendix~\ref{app:discrete} separates these effects with five fixed-active-set cases: fully clean data; noise in every data path; a noisy field with the initial slice restored to its clean value; noise only in the temporal target; and noise only in the right-hand-side library. Comparing these cases identifies which data path drives a selected-order shift.
\end{remark}

Corollary~\ref{cor:multiplicative} covers the principal experimental noise model: although multiplicative noise is heteroscedastic, the variance of a fixed linear weak feature still vanishes under refinement. Appendix~\ref{app:altnoise} shows the same weak--strong separation under additive Gaussian noise.

The weak formulation reduces noise amplification, but accurate order recovery also requires the fractional order to be distinguishable within the weak regression. The following subsection quantifies this second issue for spatial fractional orders.

\subsection{Local identifiability of spatial fractional orders}\label{subsec:local-identifiability}

On a fixed active support, we call a spatial order $\beta_j$ locally identifiable when a small change in $\beta_j$ produces a change in the weak regression that cannot be reproduced by refitting the active coefficients. This distinction matters because a weak formulation can suppress measurement noise while nearby orders still generate very similar regression columns. To isolate the order effect, fix the temporal branch, the active support, and all continuous orders except $\beta_j$. Using exact (non-interpolated) weak features, write
\begin{equation}\label{eq:local-ident-regression}
\bm b=\Tmat(\beta_j)\bm\xi+\bm r,
\qquad
\Tmat=[\bm\theta_1,\ldots,\bm\theta_c]\in\mathbb R^{K\times c}.
\end{equation}
At the reference order, $\dot{\bm\theta}_j=\partial\bm\theta_j/\partial\beta_j$ measures how the $j$th weak column changes with $\beta_j$. Let $P_{\Tmat}=\Tmat\Tmat^\dagger$ denote the orthogonal projector onto $\operatorname{col}(\Tmat)$, where $\Tmat^\dagger$ is the Moore--Penrose pseudoinverse.

\begin{proposition}[Local order sensitivity after coefficient refitting]\label{prop:local-identifiability}
Suppose $\bm\theta_j\ne0$ and $\xi_j\ne0$. For an infinitesimal perturbation $\beta_j\mapsto\beta_j+\delta\beta_j$, allow the coefficients to refit as $\bm\xi\mapsto\bm\xi+\delta\beta_j\bm v$. Then the smallest first-order change in the model mean is
\begin{equation}\label{eq:profiled-local-change}
\min_{\bm v\in\mathbb R^c}
\left\lVert\Tmat\bm v+\xi_j\dot{\bm\theta}_j\right\rVert_2
=|\xi_j|\,\left\lVert(I-P_{\Tmat})\dot{\bm\theta}_j\right\rVert_2.
\end{equation}
Hence $\beta_j$ is first-order confounded with coefficient changes when $\dot{\bm\theta}_j\in\operatorname{col}(\Tmat)$.
\end{proposition}

\begin{proof}
Taylor expansion gives
\[
\Tmat(\beta_j+\delta\beta_j)(\bm\xi+\delta\beta_j\bm v)-\Tmat(\beta_j)\bm\xi
=\delta\beta_j\bigl(\Tmat\bm v+\xi_j\dot{\bm\theta}_j\bigr)+o(|\delta\beta_j|).
\]
The term $\xi_j\dot{\bm\theta}_j$ is the first-order change caused by perturbing the order. Refitting the coefficients contributes $\Tmat\bm v$, which can reproduce any vector in $\operatorname{col}(\Tmat)$. Least squares therefore cancels the component of $\xi_j\dot{\bm\theta}_j$ that lies in this column space. The component that remains is $\xi_j(I-P_{\Tmat})\dot{\bm\theta}_j$; taking its norm gives Eq.~\eqref{eq:profiled-local-change}.
\end{proof}

This result motivates the dimensionless, coefficient-independent diagnostic
\begin{equation}\label{eq:weak-order-sensitivity}
S_{\beta_j}
=\frac{\left\lVert(I-P_{\Tmat})\dot{\bm\theta}_j\right\rVert_2}
       {\lVert\bm\theta_j\rVert_2}.
\end{equation}
The numerator measures the part of the column change that remains after the best first-order coefficient adjustment, and the denominator removes the scale of the column itself. Thus $S_{\beta_j}=0$ corresponds to complete first-order confounding on the fixed support, whereas a larger value indicates that the order produces a more distinct regression direction. The diagnostic is local and support-conditioned; in Section~\ref{subsec:mainresults} we compare it with fixed-support order errors under noise.

A simple Fourier calculation explains why this sensitivity depends on both the operator and the spectral content of the observed field. Consider one isolated periodic linear term before weak projection, and let
\[
E(\kappa)=\sum_{i=0}^{n_t-1}|\widehat u(t_i,\kappa)|^2
\]
be the total energy carried by spatial wavenumber $\kappa$ over the observed times. This quantity is relevant because an order can only be distinguished through wavenumbers that are actually present in the data. Both periodic operators satisfy $|s_\beta(\kappa)|=|\kappa|^\beta$, so the fraction of operator-weighted energy carried by each nonzero wavenumber is
\[
q_\beta(\kappa)=
\frac{|\kappa|^{2\beta}E(\kappa)}
     {\sum_{\nu\ne0}|\nu|^{2\beta}E(\nu)},\qquad \kappa\ne0.
\]
Applying the same coefficient-refitting argument to the Fourier multipliers gives the isolated-term sensitivity
\begin{equation}\label{eq:spectral-order-sensitivity}
S_{\beta,\mathrm{spec}}^2=
\begin{cases}
\operatorname{Var}_{q_\beta}[\log|\kappa|], & \text{Riesz},\\[2pt]
\operatorname{Var}_{q_\beta}[\log|\kappa|]+\pi^2/4, & \text{directional}.
\end{cases}
\end{equation}
Here $\operatorname{Var}_{q_\beta}[\log|\kappa|]$ measures the spread of the operator-weighted wavenumbers on a logarithmic scale. For a Riesz term, this spread is the entire source of local order sensitivity: if all spectral weight lies at one $|\kappa|$, changing $\beta$ only rescales the feature and can be absorbed exactly by the coefficient. With energy at several distinct wavenumbers, changing $\beta$ alters their relative magnitudes and becomes easier to distinguish. A directional derivative has the same magnitude scaling and also changes phase through $\exp(\mathrm i\pi\beta\operatorname{sgn}(\kappa)/2)$, which contributes the additional $\pi^2/4$ term for a real field with conjugate-symmetric spectral energy. Eq.~\eqref{eq:spectral-order-sensitivity} therefore explains the operator-level mechanism. The reported diagnostic uses $S_{\beta_j}$ in Eq.~\eqref{eq:weak-order-sensitivity} on the full weak design, so it also incorporates the test functions and the other active columns.

\section{Pareto-based subset selection}\label{sec:pareto}

The second core contribution is a search procedure that treats term type and fractional order jointly without constructing a dense fixed dictionary.

Algorithm~\ref{alg:pareto} summarises the mechanism of Weak-Pareto from building the weak feature library to the exact-order refit. This section details its search and selection stages.

\subsection{The weak regression system}\label{subsec:regression}

Stacking the $K$ weak equations~\eqref{eq:weak-residual}, one row per test function, gives the following system of equations:
\begin{equation}
\bm b(m_\alpha,\alpha)=\Tmat(\bm p,\bm\beta)\,\bm\xi+\bm r,
\qquad
b_k(m_\alpha,\alpha)=\inner{\Top_{m_\alpha,\alpha}u}{\phi_k},
\quad
\Tmat_{k,j}=\inner{u^{p_j}\Xop_{\beta_j}u}{\phi_k}.
\label{eq:regression}
\end{equation}
The target $\bm b$ is assembled from Eq.~\eqref{eq:caputo-weak}. Each design column is constructed from Eq.~\eqref{eq:linear-spatial} or~\eqref{eq:nonlinear-weak}. For fixed powers and orders, Eq.~\eqref{eq:regression} is linear in $\bm\xi$. Weak-Pareto therefore solves the coefficients directly and reserves global optimisation for the fractional orders.

\subsection{Best-subset regression}\label{subsec:bestsubset}

Section~\ref{subsec:encoding} represents a candidate as $\mathcal M=(m_\alpha,\alpha,\bm p,\bm\beta,\bm\xi)$. At a fixed support size $c$, choosing the discrete temporal branch $m_\alpha$ and power vector $\bm p$ leaves the temporal and spatial orders as continuous variables, while the coefficients are obtained by linear regression. Because the encoding does not predefine an order dictionary, discovery therefore becomes a sequence of best-subset problems indexed by $c$. Let $\mathcal A=\{0,\dots,P\}$ be the admitted powers. Its $c$-fold Cartesian product is $\mathcal A^c=\mathcal A\times\cdots\times\mathcal A$, and we retain only the nondecreasing power patterns
\[
\mathfrak P_c=\{(p_1,\dots,p_c)\in\mathcal A^c:p_1\le\cdots\le p_c\}.
\]
For example, if $\mathcal A=\{0,1,2\}$ and $c=2$, then $\mathfrak P_2=\{(0,0),(0,1),(0,2),(1,1),(1,2),(2,2)\}$. The ordering removes duplicate permutations of the powers; the associated spatial orders remain continuous and are optimised independently for the terms. For the declared temporal range $[\alpha_{\min},\alpha_{\max}]$ and separator $\epsilon_\alpha$, define
\begin{equation*}
\begin{aligned}
\mathcal I_{\mathrm{sub}}&=[\alpha_{\min},\alpha_{\max}]\cap(0,1-\epsilon_\alpha], &
\mathcal I_{\mathrm{int}}&=[\alpha_{\min},\alpha_{\max}]\cap\{1\},\\
\mathcal I_{\mathrm{sup}}&=[\alpha_{\min},\alpha_{\max}]\cap[1+\epsilon_\alpha,2).
\end{aligned}
\end{equation*}
The sets correspond directly to the temporal label in the encoding: $\mathcal I_{\mathrm{sub}}$ contains admitted fractional orders below one, $\mathcal I_{\mathrm{int}}$ contains the exact integer candidate $\alpha=1$, and $\mathcal I_{\mathrm{sup}}$ contains admitted fractional orders above one; empty sets are omitted. For each $c=1,\dots,c_{\max}$, we solve
\begin{equation}
\mathcal M_c^\star=\argmin_{\bm p\in\mathfrak P_c,\,m_\alpha,\,\alpha\in\mathcal I_{m_\alpha},\,\bm\beta} J_c(\bm p,m_\alpha,\alpha,\bm\beta),
\qquad
J_c=\log_{10}\!\bigl(\mathcal E_{\mathrm{val}}+\eps\bigr)+\Pi_{\mathrm{dup}},
\label{eq:best-subset}
\end{equation}
where $\mathcal E_{\mathrm{val}}$ is the variance-normalised held-out score in Eq.~\eqref{eq:val} below, and $\Pi_{\mathrm{dup}}$ penalises nearly duplicate terms with the same power. Each branch is optimised independently and scored by the same normalised criterion; changes in the scale of the branch-specific target therefore do not distort the comparison. Coefficients are fitted only on the training weak rows. With $S=\operatorname{diag}(\lVert\bm\theta_{1,\mathrm{tr}}\rVert_2,\dots,\lVert\bm\theta_{c,\mathrm{tr}}\rVert_2)$ and $\widetilde\Tmat_{\mathrm{tr}}=\Tmat_{\mathrm{tr}}S^{-1}$, the inner column-normalised ridge fit is
\begin{equation}
\widehat{\bm\xi}_{\mathrm{tr}}=S^{-1}\bigl(\widetilde\Tmat_{\mathrm{tr}}^{\!\top}\widetilde\Tmat_{\mathrm{tr}}+\lambda I\bigr)^{-1}\widetilde\Tmat_{\mathrm{tr}}^{\!\top}\bm b_{\mathrm{tr}}.
\label{eq:ridge}
\end{equation}
Appendix~\ref{subsec:objective} gives the row split, normalisations, duplicate penalty, and numerical constants. Differential evolution optimises the continuous orders within each temporal branch \cite{storn1997de}; in the integer branch, $\alpha=1$ is fixed. Powers are enumerated only as nondecreasing patterns $\mathfrak P_c$, eliminating permutation duplicates. This bi-level design confines the global search to a small order space while solving the linear coefficients at every objective evaluation.

\subsection{The parsimony-promoting elbow}\label{subsec:sweep}

Solving~\eqref{eq:best-subset} for increasing $c$ produces a sequence of best models $\mathcal M_1^\star,\mathcal M_2^\star,\dots$ and a Pareto front of validation error against support size. Model quality is scored on held-out evaluation rows,
\begin{equation}
\mathcal E_{\mathrm{val}}(\mathcal M)=\frac{n_{\mathrm{val}}^{-1}\lVert \bm b_{\mathrm{val}}-\Tmat_{\mathrm{val}}\widehat{\bm\xi}\rVert_2^2}{\operatorname{Var}(\bm b_{\mathrm{val}})+\eps},
\label{eq:val}
\end{equation}
where $\operatorname{Var}(\bm b_{\mathrm{val}})=n_{\mathrm{val}}^{-1}\sum_{i=1}^{n_{\mathrm{val}}}(b_i-\overline b_{\mathrm{val}})^2$ is the empirical validation-target variance used by the implementation, and $\eps=10^{-14}$ guards against a zero or numerically negligible denominator. Variance normalisation makes the score dimensionless and comparable across temporal targets whose scale changes with $\alpha$. Training and validation use disjoint weak rows, although overlapping test-function supports make them correlated measurements of the same noisy field. Thus $K$ counts weak equations and is distinct from an effective number of independent observations. We use $\mathcal E_{\mathrm{val}}$ for internal model selection and multi-seed experiments for across-realisation assessment.

Let $\mathcal E_c^\star$ be the best validation error at support size $c$. The sweep stops when the relative improvement falls below $\delta$ after the earliest admissible stopping size $c_{\min}=2$, or when the selected elbow remains unchanged after an additional size. To define the elbow, the implementation uses the benefit coordinate $y_c=-\log_{10}(\mathcal E_c^\star+\eps)$ together with support size $c$. After min--max normalising both coordinates, the selected interior point maximises its signed vertical excess above the chord joining the first and last points. This chord-based criterion is related to normalised difference-curve knee detection \cite{satopaa2011kneedle}. Thus ``above'' refers to the normalised benefit--complexity coordinates; when the same front is drawn as raw validation error versus support size, the corresponding knee appears below the endpoint chord. If no interior point has positive signed excess, the smallest model is retained. With only two sizes, the larger model is selected only when it improves $\log_{10}\mathcal E_c^\star$ by at least $0.15$. The two-point margin is heuristic; Appendix~\ref{subsec:elbow-margin-sensitivity} examines sensitivity over $0.10$--$0.20$.

\begin{algorithm}[t]
\caption{Weak FPDE discovery via Pareto-based subset selection}\label{alg:pareto}
\begin{algorithmic}[1]
\Require Data $u$ on a grid; test functions $\{\phi_k\}$; powers $\mathcal A=\{0,\dots,P\}$, which define $\mathfrak P_c$; max size $c_{\max}$; earliest stopping size $c_{\min}=2$; plateau tolerance $\delta$
\State Build the weak target $\bm b(m_\alpha,\alpha)$ and weak columns $\Tmat_{k,j}$ via the adjoints \eqref{eq:caputo-weak}--\eqref{eq:nonlinear-weak}
\State Split rows into training/validation
\For{$c=1,2,\dots,c_{\max}$}
  \For{each $\bm p\in\mathfrak P_c$ generated from $\mathcal A$}
     \For{each admitted temporal mode $m_\alpha$}
       \State Optimise $\alpha$ only within $\mathcal I_{m_\alpha}$ (fixed at $1$ for $m_\alpha=\mathrm{int}$) and optimise $\bm\beta$, refitting $\widehat{\bm\xi}$ by \eqref{eq:ridge}
     \EndFor
     \State Retain the lowest-objective temporal mode for this power pattern
  \EndFor
  \State $\mathcal M_c^\star\gets$ best model of size $c$ by the objective $J_c$ in \eqref{eq:best-subset}; let $\mathcal E^\star_c=\mathcal E_{\mathrm{val}}(\mathcal M_c^\star)$
  \State Recompute the current signed-chord elbow $\widehat c_{\mathrm{elbow}}$ from $\{\mathcal M_j^\star\}_{j=1}^{c}$
  \If{$c\ge c_{\min}$ and the validation-improvement plateau condition holds}
     \State \textbf{break}
  \ElsIf{$\widehat c_{\mathrm{elbow}}<c$ and the elbow is unchanged for the declared patience}
     \State \textbf{break} \Comment{selection stability}
  \EndIf
\EndFor
\State Select the elbow model $\mathcal M^\star$ of the Pareto front $\{\mathcal M_c^\star\}$
\State Prune numerically inactive terms (Section~\ref{subsec:prune}) and refit at exact orders (Section~\ref{subsec:exact})
\State \Return $\mathcal M^\star$ \Comment{representing the discovered fractional equation}
\end{algorithmic}
\end{algorithm}

\subsection{Inactive-term pruning}\label{subsec:prune}

A selected model can occasionally contain a term with negligible fitted effect. We remove such terms using a scale-aware, non-oracle rule. For $\Tmat\widehat{\bm\xi}=\sum_j\widehat\xi_j\bm\theta_j$, define $r_j=\lVert\widehat\xi_j\bm\theta_j\rVert_2/(\lVert\Tmat\widehat{\bm\xi}\rVert_2+\eps)$. Term $j$ is removed when $r_j\le\tau_{\mathrm{contrib}}$ or $|\widehat\xi_j|\le\tau_{\mathrm{abs}}$. The ratio is evaluated over all finite weak rows. Unlike raw coefficient magnitude, it accounts for the different scales of fractional-library columns. Because partially cancelling terms can make $r_j>1$, only small values are interpreted.

\subsection{Exact-order refitting}\label{subsec:exact}

During the global search, weak columns are interpolated from features precomputed on an order grid. Here ``exact-order'' means that the operators are subsequently evaluated directly at continuous order values instead of by interpolation. After elbow selection and pruning, local gradient-free steps refine the selected orders within the temporal branch chosen during the global search. If a fractional branch is selected, $\alpha$ and the non-identity spatial orders are polished within their branch bounds and local trust regions. If the exact-integer branch is selected, $\alpha$ remains fixed at $1$ and only the non-identity spatial orders are polished. The temporal-branch comparison is completed during model selection, so this conditional refit operates only within the selected branch. Let $\mathcal M^\star$ denote the selected model and $(\widehat\alpha^{\,\mathcal M^\star},\widehat{\bm\beta}^{\,\mathcal M^\star})$ its polished orders. The final coefficients are then
\begin{equation}
\widehat{\bm\xi}^{\,\mathcal M^\star}
=\argmin_{\bm\xi}\;\bigl\lVert \bm b^{\mathcal M^\star}-\Tmat^{\mathcal M^\star}\bm\xi\bigr\rVert_2^2
+\lambda\bigl\lVert S^{\mathcal M^\star}\bm\xi\bigr\rVert_2^2,
\label{eq:exact-refit}
\end{equation}
where the target and design are evaluated directly at the polished orders and $S^{\mathcal M^\star}$ column-normalises the design. The same ridge parameter is used during search and refitting. Model selection is completed before this update; the training score, validation score, objective, and heuristic information criteria therefore keep their selection-stage meanings. The final all-row residual is stored separately as $\mathcal E_{\mathrm{fit}}$. Direct evaluation removes order-grid interpolation error from this conditional refit. The selected active-term set, temporal mode, and optimisation basin can still depend on the order grid, search trajectory, order bounds, test functions, and data resolution.

\subsection{Computational cost}\label{subsec:cost}

Two choices control the cost. First, the bi-level formulation restricts global optimisation to the order variables---$c+1$ dimensions for a fractional temporal branch and $c$ for the exact integer branch---while the coefficients are solved directly. Second, early stopping usually evaluates support sizes only up to the selected model plus one.

Let $N=n_tn_x$ be the number of field samples, $K=K_tK_x$ the number of weak rows, $G_\alpha$ and $G_\beta$ the numbers of temporal and spatial order nodes, and $P+1$ the number of powers. Weak-Pareto precomputes the target and candidate features at these order nodes. At support size $c$, the number of non-redundant power patterns is $N_p(c)=\binom{P+c}{c}$. A separable projection $TUX^{\!\top}$ costs $O(K_tN+Kn_x)$, and a periodic spectral adjoint at one order costs $O(K_xn_x\log n_x)$. Precomputation is therefore linear in $G_\alpha$ and $G_\beta$, with memory $O\!\left(K[G_\alpha+(P+1)G_\beta]+N\right)$.

After precomputation, one objective evaluation forms a $K\times c$ design and solves a ridge system in $O(Kc^2+c^3)$ time and $O(Kc)$ memory. If $N_{\mathrm{DE}}(c)$ is the number of differential-evolution evaluations per power pattern and the sweep stops at $c_{\mathrm{stop}}$, the total search cost is
\[
O\!\left(\sum_{c=1}^{c_{\mathrm{stop}}}N_p(c)N_{\mathrm{DE}}(c)\,[Kc^2+c^3]\right)
\]
after precomputation. In the reported experiments $c\le4$; hence the cubic term is negligible and each objective evaluation is effectively linear in the number of weak rows. Fast Fourier transforms accelerate the dominant periodic-operator calculations.

\section{Experiments and results}\label{sec:exp}

The experiments address three primary questions. First, can Weak-Pareto recover parsimonious FPDEs across linear and nonlinear benchmarks? Second, are weak measurements more robust than pointwise fractional features under matched selection? Third, is continuous-order subset search more reliable than a dense fixed dictionary? Supporting studies examine the superunit temporal branch, a two-dimensional anisotropic example, a contemporary neural fractional-discovery framework, computational cost, and applicability to irregular experimental data.

\subsection{Empirical benchmarks and evaluation metrics}\label{subsec:setup}

We use four periodic main benchmarks (Table~\ref{tab:benchmarks}): FADE, two Riesz reaction--diffusion (RD) equations with integer or fractional time, and a nonlinear fractional Burgers equation. The Burgers case is the principal nonlinear benchmark because its support contains the genuine quadratic transport term $-u\,\partial_xu$; it also tests discrimination between fractional diffusion of order $1.7$ and the nearby integer second derivative. Section~\ref{subsec:superunit} reports a separate fixed-support superunit diagnostic, and Appendix~\ref{app:challenging} adds an integer-only equation and a case with two fractional spatial derivatives. Unless stated otherwise, noise is multiplicative, $u\mapsto u(1+\rho\zeta)$ with $\zeta\sim\mathcal U[-1,1]$, and tables report $100\rho$ percent. Recovery counts are over five seeds, order and coefficient errors are conditioned on support/power recovery, and $\mathcal E_{\mathrm{fit}}$ is summarised over all seeds. The $10\%$ Riesz cases are deliberately severe identifiability tests at the present resolution. The time--space reaction--diffusion generator and evaluator share the Caputo L1 discretisation; therefore, the clean experiment is a discretisation-consistency test rather than an independent forward-solver validation.

The main experiments use one common, non-oracle configuration: automatic support-size stopping, $c_{\max}=4$, candidate powers $p\in\{0,1,2\}$, and elbow selection. Appendix~\ref{app:challenging} examines sensitivity to these choices.

\paragraph{Evaluation protocol.} We separate structural recovery from parameter accuracy. \emph{Support/power recovery} requires the correct number of terms and matching integer powers after pruning. \emph{Operator-structure recovery} additionally requires the correct temporal branch, spatial operator modes, and an absolute error no greater than $\tau_q=0.15$ for every positive derivative order; identity terms must be identified exactly. A fractional time order near one is not the exact operator $\partial_t$, and a low-order Riesz term is not the reaction identity.

The tolerance is a predefined operational criterion. For the main-text synthetic benchmarks, the positive true orders range from approximately $0.8$ to $2.0$; hence $\tau_q=0.15$ corresponds to relative deviations of $7.5\%$--$18.75\%$. An absolute criterion is appropriate because fractional order is dimensionless and temporal and spatial orders are measured on the same scale. Rescoring the 60 noisy runs in the matched weak-versus-strong comparisons gives 45, 47, and 47 recoveries for Weak-Pareto at $\tau_q=0.125$, $0.15$, and $0.175$, respectively; the strong-form method gives 0/60 at all three values. Thus, the main comparison is insensitive to moderate changes around the reported tolerance. We report continuous order errors alongside the binary counts because the cut-off is an interpretation rule, not a measure of uncertainty.

Parameter accuracy is measured by $e_\alpha=|\widehat\alpha-\alpha^\star|$. To compare right-hand-side terms, we define a one-to-one matching map $\pi$ from true-term indices to selected-term indices. Processing the true terms in their stored order, $\pi(j)$ assigns term $j$ to the nearest unmatched selected order with the same power. With $\widehat{\bm\xi}_{\pi}=(\widehat\xi_{\pi(1)},\dots,\widehat\xi_{\pi(c)})$ and $\eps_\xi=10^{-12}$, we use
\begin{equation*}
 e_\beta^{\max}=\max_j|\widehat\beta_{\pi(j)}-\beta_j^\star|,\qquad
 e_\xi^{\max}=\max_j\frac{|\widehat\xi_{\pi(j)}-\xi_j^\star|}{|\xi_j^\star|+\eps_\xi},\qquad
 e_{\xi,2}=\frac{\lVert\widehat{\bm\xi}_{\pi}-\bm\xi^\star\rVert_2}{\lVert\bm\xi^\star\rVert_2+\eps_\xi}.
\end{equation*}
These errors are averaged only over runs with correct support and powers; $e_{\xi,2}$ is especially useful when a true coefficient is small. Model selection uses the held-out, variance-normalised score $\mathcal E_{\mathrm{val}}$. After selection, the model is refit on all weak rows and reported with $\mathcal E_{\mathrm{fit}}=\lVert\bm b-\Tmat\widehat{\bm\xi}\rVert_2/(\lVert\bm b\rVert_2+\eps)$. Because the weak and strong frameworks use different regression rows, $\mathcal E_{\mathrm{fit}}$ is a within-framework diagnostic; cross-method conclusions rely on recovery rates and parameter errors.

\begin{table}[t]
\caption{Main benchmark equations in the encoding of Eq.~\eqref{eq:model}; $\Rop_\beta$ is the Riesz operator and $(0,0)$ denotes the identity term $u$. The spatial operator is directional for FADE and fractional Burgers and Riesz for the two reaction--diffusion benchmarks. True terms are listed directly as triples $(p,\beta,\xi)$; horizons $(T,L_x)$ are approximately $(15,30)$, $(1,20)$, $(3,20)$, and $(12,30)$, respectively.}\label{tab:benchmarks}
\footnotesize
\setlength{\tabcolsep}{4pt}
\begin{tabular*}{\textwidth}{@{\extracolsep{\fill}}lllc}
\toprule
Benchmark & Governing equation & True terms $(p,\beta,\xi)$ & Grid $n_t{\times}n_x$ \\
\midrule
FADE  & $D_t^{0.8}u=-\partial_x u+0.5\,D_x^{1.7}u$ & $(0,1,-1),\;(0,1.7,0.5)$ & $150{\times}120$ \\
Frac.\ RD (space) & $\partial_t u=0.04\,u+0.18\,\Rop_{1.65}u$ & $(0,0,0.04),\;(0,1.65,0.18)$ & $90{\times}96$ \\
Frac.\ RD (time--space) & $D_t^{0.82}u=0.03\,u+0.12\,\Rop_{1.55}u$ & $(0,0,0.03),\;(0,1.55,0.12)$ & $80{\times}80$ \\
Frac.\ Burgers & $\partial_t u=-u\,\partial_x u+0.25\,D_x^{1.7}u$ & $(1,1,-1),\;(0,1.7,0.25)$ & $150{\times}120$ \\
\botrule
\end{tabular*}
\end{table}

\subsection{Recovery accuracy of Weak-Pareto}\label{subsec:mainresults}

Table~\ref{tab:main} reports Weak-Pareto at $10\%$ multiplicative noise, used as a common stress level across the four benchmarks. FADE and fractional Burgers are recovered in all five seeds, including the operator orders, with small spatial-order and coefficient-vector errors. The Riesz reaction--diffusion cases are more difficult. Support and powers are recovered in $3/5$ space-fractional runs and all five time--space runs, but no run satisfies the full operator criterion at $10\%$ noise. Order identification is the more persistent difficulty, although support recovery also degrades in the space-fractional case at $10\%$ noise; the worst relative coefficient error is amplified further by the small reaction coefficients. Table~\ref{tab:rdnoise} resolves these effects across lower noise levels.

\begin{table}[t]
\caption{Weak-Pareto on the four main benchmarks at $10\%$ multiplicative noise. Recovery definitions and reporting conventions follow the evaluation protocol in Section~\ref{subsec:setup}. The relative $e_\xi^{\max}$ is inflated on the reaction--diffusion benchmarks by their small reaction coefficient, for which $e_{\xi,2}$ is more informative. Temporal-bound concentration is detailed in Table~\ref{tab:rdnoise}.}\label{tab:main}
\footnotesize
\setlength{\tabcolsep}{0pt}
\begin{tabular*}{\textwidth}{@{\extracolsep\fill}lccccccc}
\toprule
Benchmark & \shortstack{Support/\\power} & \shortstack{Operator\\structure} & $e_\alpha$ & $e_\beta^{\max}$ & $e_\xi^{\max}$ & $e_{\xi,2}$ & $\mathcal E_{\mathrm{fit}}$ \\
\midrule
FADE & 5/5 & 5/5 & $0.002\pm0.002$ & $0.07\pm0.03$ & $0.12\pm0.09$ & $0.08\pm0.05$ & $0.022\pm0.001$ \\
\shortstack[l]{Frac. RD\\(space)} & 3/5 & 0/5 & $0.20\pm0.00$ & $0.37\pm0.15$ & $2.56\pm1.97$ & $0.60\pm0.43$ & $0.67\pm0.09$ \\
\shortstack[l]{Frac. RD\\(time--space)} & 5/5 & 0/5 & $0.17\pm0.01$ & $0.54\pm0.38$ & $3.00\pm2.45$ & $0.94\pm0.76$ & $0.39\pm0.04$ \\
Burgers & 5/5 & 5/5 & $0.000\pm0.000$ & $0.003\pm0.002$ & $0.005\pm0.004$ & $0.002\pm0.002$ & $0.051\pm0.007$ \\
\botrule
\end{tabular*}
\end{table}

Table~\ref{tab:rdnoise} now includes the clean reference for both Riesz reaction--diffusion benchmarks. At $0\%$ noise, both recover the complete operator in all five seeds, confirming that the subsequent failures are noise-induced. The clean time--space $e_\alpha$ standard deviation is $1.4\times10^{-6}$ before rounding, although it appears as $0.0000$ at the precision used in Table~\ref{tab:rdnoise}. Lowering the positive noise level improves the Riesz-order estimates, but complete operator recovery remains difficult because it also requires the correct temporal mode and the reaction identity. The time--space case reaches $2/5$ operator recoveries at $2\%$ noise; no positive-noise row for the space-fractional case does. Boundary attainment is clearest for the space-fractional case: all five support/power-recovered runs at $5\%$ noise and all three runs entering the conditioned $10\%$ summary have $e_\alpha=0.20\pm0.00$. Since the true order is $\alpha=1.00$ and $\alpha_{\min}=0.80$, this is the truncation distance; zero dispersion therefore indicates constraint saturation, not precision. In the time--space case, $e_\alpha$ rises from $0.07\pm0.03$ at $2\%$ noise to $0.17\pm0.01$ at $10\%$, indicating increasing concentration near the lower boundary. Order errors are conditioned on support/power recovery. Appendix~\ref{app:discrete} identifies the noisy temporal target as the dominant source of the shift in these profiles. The positive-noise rows should therefore be read as support recovery with increasingly accurate spatial order at lower noise.

\begin{table}[t]
\caption{Order identifiability of the two Riesz reaction--diffusion benchmarks versus noise. Reporting conventions follow the evaluation protocol in Section~\ref{subsec:setup}.}\label{tab:rdnoise}
\footnotesize
\setlength{\tabcolsep}{0pt}
\begin{tabular*}{\textwidth}{@{\extracolsep\fill}lccccccc}
\toprule
Benchmark & Noise (\%) & \shortstack{Support/\\power} & \shortstack{Operator\\structure} & $e_\alpha$ & $e_\beta^{\max}$ & $e_\xi^{\max}$ & $\mathcal E_{\mathrm{fit}}$ \\
\midrule
\multirow{4}{*}{\shortstack[l]{Frac. RD\\(space)}} & 0 & 5/5 & 5/5 & $0.0000\pm0.0000$ & $0.0006\pm0.0001$ & $0.0060\pm0.0001$ & $0.0042\pm0.0001$ \\
 & 2 & 5/5 & 0/5 & $0.19\pm0.01$ & $0.19\pm0.12$ & $1.58\pm1.36$ & $0.18\pm0.02$ \\
 & 5 & 5/5 & 0/5 & $0.20\pm0.00$ & $0.30\pm0.15$ & $2.03\pm1.79$ & $0.37\pm0.04$ \\
 & 10 & 3/5 & 0/5 & $0.20\pm0.00$ & $0.37\pm0.15$ & $2.56\pm1.97$ & $0.67\pm0.09$ \\
\midrule
\multirow{4}{*}{\shortstack[l]{Frac. RD\\(time--space)}} & 0 & 5/5 & 5/5 & $0.0008\pm0.0000$ & $0.0010\pm0.0001$ & $0.0077\pm0.0001$ & $0.0052\pm0.0001$ \\
 & 2 & 5/5 & 2/5 & $0.07\pm0.03$ & $0.11\pm0.11$ & $1.42\pm1.21$ & $0.10\pm0.01$ \\
 & 5 & 5/5 & 0/5 & $0.15\pm0.03$ & $0.34\pm0.21$ & $2.57\pm1.38$ & $0.22\pm0.02$ \\
 & 10 & 5/5 & 0/5 & $0.17\pm0.01$ & $0.54\pm0.38$ & $3.00\pm2.45$ & $0.39\pm0.04$ \\
\botrule
\end{tabular*}
\end{table}

\paragraph{Support-conditioned spatial-order diagnostic.}
We next test whether the local sensitivity in Eq.~\eqref{eq:weak-order-sensitivity} is consistent with the spatial-order errors above. For each main benchmark, $S_\beta$ is evaluated on the clean field and true support using the exact weak features; the temporal branch and all other orders are fixed at their true values. In a separate fixed-support profile under multiplicative noise, only the principal noninteger linear spatial order is varied, while all coefficients are refitted on the same training rows with the ridge parameter and variance-normalised validation score used by Weak-Pareto. This removes support-selection and temporal-order errors from the diagnostic. Table~\ref{tab:local-identifiability} reports the results.

\begin{table}[t]
\centering
\caption{Support-conditioned local sensitivity $S_\beta$ and fixed-support absolute error $e_{\beta,\mathrm{fix}}$ of the principal noninteger linear spatial order (five seeds). The complete $0\%$, $2\%$, $5\%$, and $10\%$ profiles and per-seed estimates are provided in Online Resource 1.}\label{tab:local-identifiability}
\footnotesize
\setlength{\tabcolsep}{3.0pt}
\begin{tabular*}{\textwidth}{@{\extracolsep\fill}llccc}
\toprule
Benchmark & Operator & $S_\beta$ & $e_{\beta,\mathrm{fix}}$ ($0\%$) & $e_{\beta,\mathrm{fix}}$ ($10\%$) \\
\midrule
FADE & directional & $0.502$ & $0.0024\pm0.0054$ & $0.0278\pm0.0173$ \\
Frac.\ RD (space) & Riesz & $0.118$ & $0.0013\pm0.0009$ & $0.2117\pm0.0984$ \\
Frac.\ RD (time--space) & Riesz & $0.096$ & $0.0016\pm0.0010$ & $0.4253\pm0.3641$ \\
Frac.\ Burgers & directional & $1.610$ & $(1.3\pm2.1){\times}10^{-5}$ & $0.0027\pm0.0016$ \\
\botrule
\end{tabular*}
\end{table}

The sensitivity ordering is the reverse of the noisy fixed-support error ordering: fractional Burgers has the largest $S_\beta$ and smallest $10\%$ error, followed by FADE, whereas the two Riesz benchmarks have much smaller sensitivities and substantially larger errors. The same qualitative ordering appears in the complete-discovery $e_\beta^{\max}$ values in Tables~\ref{tab:main} and~\ref{tab:rdnoise}. All four fixed-support profiles remain accurate without noise. The four-benchmark ordering therefore provides an explanatory consistency check: $S_\beta$ characterises local, support-conditioned susceptibility to spatial-order perturbations. Caputo endpoint sensitivity and dense-dictionary collinearity arise from separate mechanisms analysed elsewhere.

\subsection{Superunit temporal-order diagnostic}\label{subsec:superunit}

The main benchmark suite contains no true temporal order in $(1,2)$. We therefore test the superunit branch on semi-analytic periodic data satisfying
\begin{equation}
{}_0^C D_t^{1.65}u=0.12\,D_x^2u,\qquad \partial_tu(0,x)=0.
\label{eq:superunit-diagnostic}
\end{equation}
For a spatial Fourier mode with wavenumber $\kappa$, the mode amplitude evolves as $E_{1.65,1}(-0.12\kappa^2t^{1.65})$, where $E_{a,b}(z)=\sum_{m=0}^{\infty}z^m/\Gamma(am+b)$ is the two-parameter Mittag--Leffler function. This semi-analytic evolution is independent of the L1 temporal discretisation used by the discovery evaluator. To isolate temporal-branch and order recovery, the support is fixed to one linear term ($c=1$, $p_1=0$), while both methods search the temporal branch, $\alpha$, and $\beta$ and fit the coefficient. This is therefore a branch/order diagnostic. We use the same noisy realisation for the weak and strong frameworks at each of five seeds. Operator recovery requires the superunit branch together with $|\widehat\alpha-1.65|\le\tau_q$ and $|\widehat\beta-2|\le\tau_q$. Table~\ref{tab:superunit} reports the resulting branch and operator recoveries.

\begin{table}[t]
\caption{Fixed-support superunit diagnostic for Eq.~\eqref{eq:superunit-diagnostic} under multiplicative-uniform noise. Complete per-seed estimates and mean $\pm$ sample-standard-deviation errors are provided in Online Resource 1.}\label{tab:superunit}
\footnotesize
\setlength{\tabcolsep}{5pt}
\begin{tabular*}{\textwidth}{@{\extracolsep\fill}lccc}
\toprule
Method & Noise (\%) & Branch recovery & Operator recovery \\
\midrule
\multirow{3}{*}{Weak-Pareto} & 0   & 5/5 & 5/5 \\
 & 0.5 & 5/5 & 5/5 \\
 & 1.0 & 5/5 & 1/5 \\
\midrule
\multirow{3}{*}{Strong-Pareto} & 0   & 5/5 & 5/5 \\
 & 0.5 & 0/5 & 0/5 \\
 & 1.0 & 0/5 & 0/5 \\
\botrule
\end{tabular*}
\end{table}

Both methods recover the clean superunit operator in all five runs; for Weak-Pareto, the clean temporal- and spatial-order errors are below $10^{-3}$ and the mean relative coefficient error is $0.028$. At $0.5\%$ noise, Weak-Pareto retains 5/5 branch and operator recovery, with $e_\alpha=0.074\pm0.038$, $e_\beta=0.037\pm0.028$, and $e_\xi=0.073\pm0.064$, whereas the strong-form comparator selects the wrong temporal branch in every run. At $1\%$, Weak-Pareto still selects the superunit branch in all five runs, but four estimates reach the upper search bound $\alpha=1.85$, leaving only 1/5 complete operator recoveries and $e_\alpha=0.172\pm0.062$. The diagnostic therefore verifies that the branch-aware weak search extends to $\alpha>1$ under clean and mild noise. The upper-bound saturation at $1\%$ is consistent with the endpoint sensitivity analysed in Remark~\ref{rem:temporal-target}.

\subsection{Robustness relative to a strong-form library}\label{subsec:robustness}

Fig.~\ref{fig:robustness} and Table~\ref{tab:robustness} compare Weak-Pareto with a strong-form fractional library under the same best-subset framework on FADE and fractional Burgers. The methods are comparable without noise. Under multiplicative noise, however, Weak-Pareto recovers the correct support in all five seeds at every tested level up to $20\%$ on both benchmarks. The strong-form framework does not recover the correct Burgers support in any noisy run and in all but the $5\%$ FADE condition, where it recovers $4/5$ supports but has order-one coefficient error ($e_\xi^{\max}\approx1.1$). Weak-Pareto's Burgers coefficient error remains below $0.022$ throughout; on FADE it remains small through $10\%$ noise and rises only at $20\%$, when the smaller diffusion coefficient becomes difficult to estimate.

The complete-framework comparison includes Weak-Pareto's exact-order refinement. The library-only contrast in Section~\ref{subsec:ablation} instead compares Strong-Pareto with Weak-Pareto without polishing under the same selector, isolating pointwise versus weak candidate measurements. Operator recovery then changes from $0/5$ to $5/5$ at $10\%$ FADE noise, identifying the weak library as the primary source of robustness. Appendix~\ref{app:altnoise} reaches the same conclusion under additive Gaussian noise: Weak-Pareto recovers the correct support in all five runs and four of five complete operators on both benchmarks, while the strong-form framework achieves no correct-support recovery.

Appendix~\ref{app:yu} provides a complementary comparison with the adapted neural fractional-discovery framework of Yu et al.~\cite{yu2025fde} on the advection--diffusion benchmark. We treat it as a method-level comparison because the two methods use different fractional-operator realisations; the matched weak--strong experiment provides the operator-controlled ablation.

\begin{figure}[t]
\centering
\includegraphics[width=\textwidth]{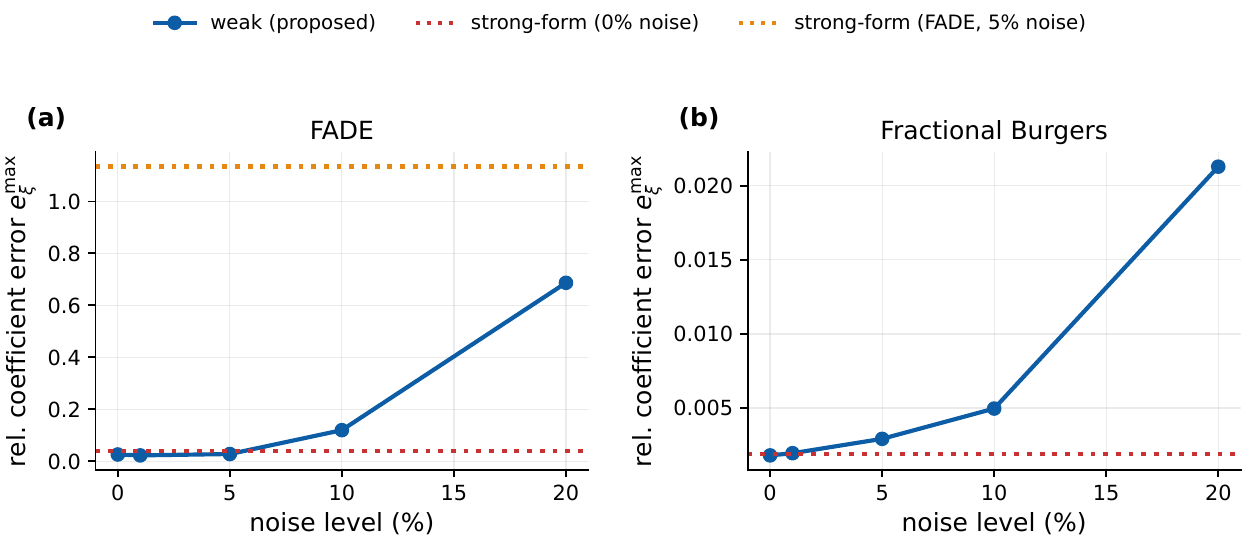}
\caption{Relative coefficient error $e_\xi^{\max}$ versus noise on (a) FADE and (b) fractional Burgers for the weak and strong-form libraries under the same selector; lower is better. The solid curve shows Weak-Pareto. Dotted horizontal lines show the strong-form result at $0\%$ noise in both panels and its additional recovered case on FADE at $5\%$ noise}\label{fig:robustness}
\end{figure}

\begin{table}[t]
\caption{Weak-Pareto versus the strong-form pointwise library under the same best-subset selector on FADE. The fit residual is a within-framework diagnostic.}\label{tab:robustness}
\begin{tabular*}{\textwidth}{@{\extracolsep\fill}lcccccc}
\toprule
& \multicolumn{3}{@{}c@{}}{Weak (proposed)} & \multicolumn{3}{@{}c@{}}{Strong-form} \\\cmidrule{2-4}\cmidrule{5-7}
Noise (\%) & Support/power & $e_\xi^{\max}$ & $\mathcal E_{\mathrm{fit}}$ & Support/power & $e_\xi^{\max}$ & $\mathcal E_{\mathrm{fit}}$ \\
\midrule
0  & 5/5 & 0.03 & $1.5{\times}10^{-3}$ & 5/5 & 0.04 & $2.7{\times}10^{-3}$ \\
1  & 5/5 & 0.02 & $2.7{\times}10^{-3}$ & 0/5 & -- & $2.7{\times}10^{-1}$ \\
5  & 5/5 & 0.03 & $1.1{\times}10^{-2}$ & 4/5 & 1.13 & $6.3{\times}10^{-1}$ \\
10 & 5/5 & 0.12 & $2.2{\times}10^{-2}$ & 0/5 & -- & $8.4{\times}10^{-1}$ \\
20 & 5/5 & 0.69 & $4.3{\times}10^{-2}$ & 0/5 & -- & $8.2{\times}10^{-1}$ \\
\botrule
\end{tabular*}
\end{table}

The matched library-only ablation discussed in Section~\ref{subsec:ablation} provides the direct evidence for the weak formulation: with the selector held fixed, replacing pointwise features by weak measurements changes FADE operator recovery at $10\%$ noise from $0/5$ to $5/5$. The additive-Gaussian experiment in Appendix~\ref{app:altnoise} confirms that this advantage is not tied to the multiplicative-uniform noise model.

\subsection{Model selection via Pareto-based subset selection}\label{subsec:paretoexp}

Fig.~\ref{fig:pareto} and Table~\ref{tab:progress} illustrate the support-size search on FADE at $5\%$ noise. Validation error drops by more than an order of magnitude from one to the true two-term support and improves only marginally at three terms. The elbow therefore selects $c=2$, and selection stability stops the search after evaluating $c=3$, before the permitted maximum $c_{\max}=4$. The procedure thus expresses parsimony through an explicit, data-driven stopping rule.

\begin{figure}[t]
\centering
\includegraphics[width=0.7\textwidth]{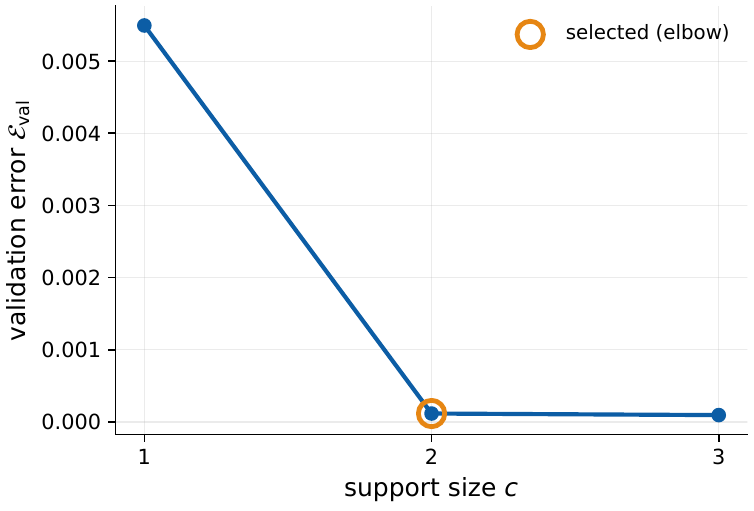}
\caption{Validation error $\mathcal{E}_{\mathrm{val}}$ versus support size on FADE ($5\%$ noise); the marked elbow ($c=2$) is the selected model. The search halts automatically after $c=3$; larger models up to $c_{\max}=4$ are therefore not explored}\label{fig:pareto}
\end{figure}

\begin{table}[t]
\caption{Support-size progress on FADE ($5\%$ noise): training and validation error of the best model at each support size. The signed elbow selects $c=2$.}\label{tab:progress}
\begin{tabular*}{\textwidth}{@{\extracolsep\fill}lccc}
\toprule
Support size $c$ & Training error & Validation error & Selected \\
\midrule
1 & $5.94{\times}10^{-3}$ & $5.49{\times}10^{-3}$ & \\
2 & $1.21{\times}10^{-4}$ & $1.16{\times}10^{-4}$ & $\checkmark$ \\
3 & $1.05{\times}10^{-4}$ & $9.40{\times}10^{-5}$ & \\
\botrule
\end{tabular*}
\end{table}

\subsection{A nonlinear fractional benchmark}\label{subsec:burgersexp}

The fractional Burgers equation $\partial_tu=-u\,\partial_xu+0.25D_x^{1.7}u$ is the principal test of nonlinear discovery because its support contains the genuine quadratic transport term $-u\,\partial_xu$. It also tests whether the framework distinguishes the fractional diffusion term $D_x^{1.7}u$ from plausible integer-order and nonlinear alternatives. We compare the true two-term structure with competing two-term structures assembled from $u_x$, $u^2u_x$, $u_{xx}$, $uD^{1.7}u$, and $D^{0.5}u$. For each structure, coefficients are refit at the exact candidate orders and scored by the relative weak residual. The residual margin is the smallest residual among the structures that do not match the ground truth, divided by the true-structure residual. Weak-Pareto selects the true pair at all six tested noise levels ($0\%$, $5\%$, $10\%$, $15\%$, $20\%$, and $25\%$). The margin decreases monotonically from $218\times$ without noise to $3.2\times$ at $25\%$ noise but remains above one, while the coefficient error stays below $0.02$. The closest competing structure replaces fractional diffusion by $u_{xx}$ at every level, showing that the weak library continues to distinguish both the noninteger diffusion order and the nonlinear transport term under substantial noise. Fig.~\ref{fig:burgers} shows the complete six-point residual curve, whereas Table~\ref{tab:burgers} reports the representative $0\%$, $10\%$, and $25\%$ rows together with coefficient errors.

\begin{table}[t]
\caption{Representative noise levels for nonlinear fractional Burgers. The residual margin is the smallest residual among the structures that do not match the ground truth, divided by the true-structure residual; $e_\xi^{\max}$ is the relative coefficient error of the true structure.}\label{tab:burgers}
\begin{tabular*}{\textwidth}{@{\extracolsep\fill}lccc}
\toprule
Noise (\%) & True structure selected & Residual margin to closest competing structure & $e_\xi^{\max}$ \\
\midrule
0  & yes & $218.3\times$ & 0.002 \\
10 & yes & $7.6\times$ & 0.007 \\
25 & yes & $3.2\times$ & 0.019 \\
\botrule
\end{tabular*}
\end{table}

\begin{figure}[t]
\centering
\includegraphics[width=\textwidth]{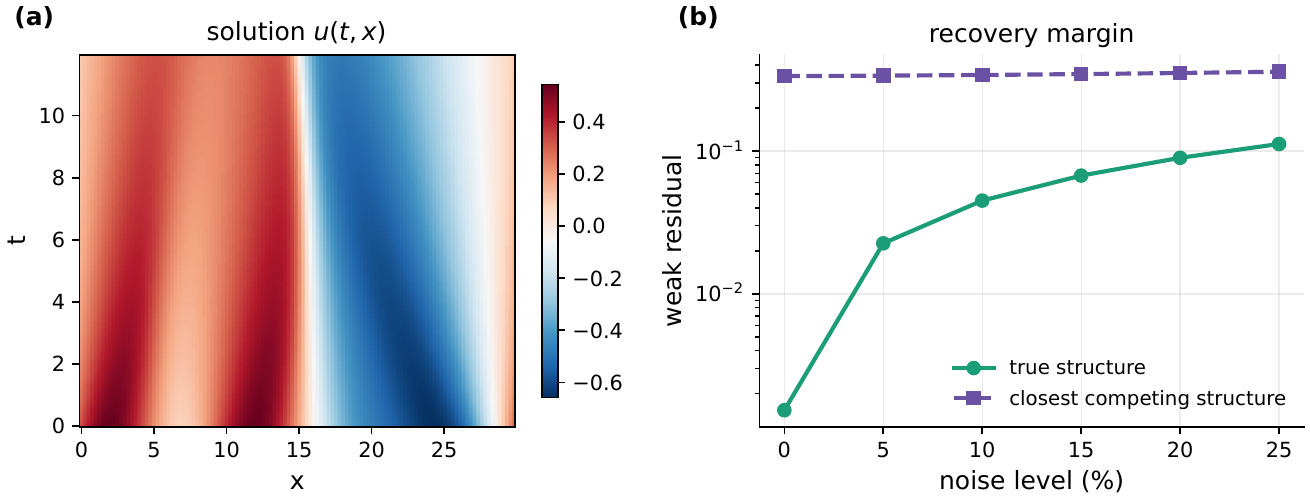}
\caption{(a) Space--time solution field of the nonlinear fractional Burgers benchmark. (b) Weak residual of the true two-term structure and the closest competing structure at $0\%$, $5\%$, $10\%$, $15\%$, $20\%$, and $25\%$ noise on a logarithmic scale; the true structure remains separated throughout}\label{fig:burgers}
\end{figure}

\subsection{Experimental frozen-soil creep}\label{subsec:frozen-soil}

We next test the weak-form principle on naturally noisy, irregularly sampled clay and silt creep measurements from Yu et al.~\cite{yu2025fde}. After averaging one repeated silt timestamp, the records contain 56 and 43 strain observations under constant loads of $1.11$ and $1.14$ MPa, respectively. We use the fractional Kelvin model
\begin{equation}
{}^{C}D_t^\alpha\epsilon(t)=\frac{\sigma_{\mathrm{load}}}{\eta_{\mathrm K}}-\frac{E}{\eta_{\mathrm K}}\epsilon(t),
\qquad 0<\alpha<1,
\label{eq:kelvin}
\end{equation}
where $\epsilon$ is strain, $E$ is the elastic modulus, $\eta_{\mathrm K}$ is the Kelvin viscosity parameter, and $\sigma_{\mathrm{load}}$ is the applied constant stress. Applying the fractional integral gives the smoothing representation
\begin{equation}
\epsilon(t)-\epsilon(0)=a_0\frac{t^\alpha}{\Gamma(\alpha+1)}+a_1 I_t^\alpha\epsilon(t),
\qquad a_0=\frac{\sigma_{\mathrm{load}}}{\eta_{\mathrm K}},\quad a_1=-\frac{E}{\eta_{\mathrm K}}.
\label{eq:kelvin-integral}
\end{equation}
For each trial order $\alpha$, the linear parameters $a_0$ and $a_1$ are fitted by least squares, and a bounded one-dimensional search selects $\alpha$. Shape-preserving interpolation is used only for quadrature; no pointwise fractional derivative is evaluated. Because the units of $\eta_{\mathrm K}$ depend on $\alpha$, Table~\ref{tab:frozen-soil} also reports the dimensionally comparable retardation time $t_{\mathrm{ret}}=(\eta_{\mathrm K}/E)^{1/\alpha}$. The silt estimate reproduces the reference order and time scale closely. For clay, the order differs by $0.098$ and $t_{\mathrm{ret}}$ by about $78\%$, indicating substantial parameter uncertainty despite a small integral residual. This section is parameter identification within the prescribed Kelvin support.

\begin{table}[t]
\centering
\caption{Integral-form identification of the fractional Kelvin model from naturally noisy frozen-soil creep records. ``Reference'' denotes the published constitutive fit. Here $E$ is in MPa, $\eta_{\mathrm K}$ has units MPa$\cdot$time$^{\alpha}$, and $t_{\mathrm{ret}}=(\eta_{\mathrm K}/E)^{1/\alpha}$ is in the time unit of the source data; therefore raw $\eta_{\mathrm K}$ values at different $\alpha$ should not be compared directly. The final column is the relative residual of Eq.~\eqref{eq:kelvin-integral}.}\label{tab:frozen-soil}
\footnotesize
\setlength{\tabcolsep}{2.0pt}
\begin{tabular*}{\textwidth}{@{\extracolsep\fill}llccccc}
\toprule
Soil & Fit & $\alpha$ & $\eta_{\mathrm K}$ & $E$ (MPa) & $t_{\mathrm{ret}}$ & Integral residual \\
\midrule
\multirow{2}{*}{Clay} & reference & 0.371 & 0.1030 & 0.320 & 0.047 & -- \\
      & weak integral & 0.469 & 0.1034 & 0.331 & 0.084 & 0.035 \\
\multirow{2}{*}{Silt} & reference & 0.562 & 0.0778 & 0.442 & 0.045 & -- \\
     & weak integral & 0.562 & 0.0744 & 0.430 & 0.044 & 0.041 \\
\botrule
\end{tabular*}
\end{table}

\subsection{Extension to two spatial dimensions}\label{subsec:twod}

The preceding discovery benchmarks use one spatial coordinate. To test whether the candidate encoding is tied to that setting, we extend each linear spatial term by a direction label $d_j\in\{x,y\}$,
\begin{equation}
{}^{C}_{0}D_t^\alpha u=\sum_{j=1}^{c}\xi_j\,\Xop_{\beta_j}^{(d_j)}u.
\label{eq:twod-model}
\end{equation}
The temporal and two spatial test bases remain separable, and the discrete adjoint is applied along the coordinate named by $d_j$. Mode-wise tensor contractions avoid assembling a dense spatial Kronecker matrix. The Pareto search, elbow rule, and exact-order refit are otherwise unchanged. The present example restricts the admitted powers to $p=0$ and focuses on directional encoding with linear terms.

We use two doubly periodic benchmarks on $[0,2\pi)^2$:
\begin{align}
{}^{C}_{0}D_t^{0.85}u
 &=0.30D_x^{1.70}u+0.20D_y^{1.40}u, \label{eq:twod-a}\\
{}^{C}_{0}D_t^{0.85}u
 &=-0.60\,\partial_xu+0.30D_x^{1.70}u+0.20D_y^{1.40}u. \label{eq:twod-b}
\end{align}
Benchmark~\eqref{eq:twod-b} requires the selector to separate two orders in the $x$ direction while assigning a third order to $y$. Each Fourier mode is propagated semi-analytically as $\widehat u(\bm\kappa,t)=\widehat u(\bm\kappa,0)E_{\alpha,1}(\lambda(\bm\kappa)t^\alpha)$, independently of the L1 evaluator used by discovery. On temporal grids with $n_t=90,179,$ and $357$, an independent L1 residual is $6.02\times10^{-3}$, $2.66\times10^{-3}$, and $1.18\times10^{-3}$, respectively, giving an observed rate $1.18$ close to the expected $2-\alpha=1.15$.

The reported grid is $90\times80\times80$. Applying the paper test-count rule independently to both spatial axes gives $30\times40\times40=48{,}000$ overlapping weak equations from the same dense data tensor. Their tensor-product refinement is computationally tractable because the example reduces repeated objective evaluations to precomputed Gram tables. We use multiplicative-uniform noise at $0\%,1\%,5\%,10\%,$ and $20\%$, $c_{\max}=4$, $\alpha\in[0.55,1.25]$, $\beta\in[0.50,2.50]$, and the same optimisation and recovery rules as Section~\ref{subsec:setup}. Boundary trimming is not used. Table~\ref{tab:twod} reports recovery across all 25 noise--seed combinations.

\begin{table}[t]
\centering
\caption{Two-dimensional directional discovery over five seeds and five noise levels ($25$ runs per benchmark). Support/direction recovery requires the correct support size and direction multiset; complete operator recovery additionally requires all orders to satisfy the protocol of Section~\ref{subsec:setup}.}\label{tab:twod}
\footnotesize
\setlength{\tabcolsep}{5pt}
\begin{tabular*}{\textwidth}{@{\extracolsep\fill}lccc}
\toprule
Benchmark & Active terms & Support/direction & Operator \\
\midrule
A: anisotropic diffusion & 2 & $25/25$ & $25/25$ \\
B: advection--diffusion & 3 & $25/25$ & $25/25$ \\
\botrule
\end{tabular*}
\end{table}

Both benchmarks retain complete direction and operator recovery throughout the noise sweep. At $20\%$ noise, $(e_\alpha,e_\beta^{\max},e_\xi^{\max})$ is $(0.00081\pm0.00054,0.00217\pm0.00140,0.0111\pm0.0023)$ for Benchmark~\eqref{eq:twod-a} and $(0.00207\pm0.00133,0.01221\pm0.00259,0.04798\pm0.00890)$ for Benchmark~\eqref{eq:twod-b}. A second spatial resolution, $90\times112\times112$, increases the number of weak rows to $94{,}080$ and again gives $25/25$ recoveries for Benchmark~\eqref{eq:twod-a}. At $20\%$ noise, $e_\xi^{\max}$ decreases from $0.0111$ to $0.0084$, while the order errors remain of comparable magnitude; we therefore use this result as a grid-stability check. Appendix~\ref{app:twod} reports the resolution and window-width diagnostics. The experiment establishes that the encoding and tensor construction extend to coordinate-dependent orders.

\FloatBarrier
\subsection{Runtime and computational cost}\label{subsec:runtime}

Table~\ref{tab:runtime} reports uncached, serial runtimes and search budgets. The main one-off cost is precomputing the order-indexed weak features; each subsequent differential-evolution evaluation solves a small ridge problem on the $K$ weak rows. On the tested benchmarks, both weak and strong frameworks run in a few seconds per seed. Measurements were obtained in CPU-only mode with Python~3.11.15 on an Apple M4 Pro system with 64~GB unified memory. They are descriptive implementation-level timings. Appendix~\ref{app:yu} uses the same CPU environment for the neural fractional-discovery framework, allowing a fair wall-clock comparison for those particular implementations.

\begin{table}[t]
\caption{Runtime and search budget on three representative benchmarks (uncached serial single-seed runs; $K$ weak rows; differential-evolution (DE) population multiplier $\times$ generations). Times are wall-clock seconds.}\label{tab:runtime}
\begin{tabular*}{\textwidth}{@{\extracolsep\fill}lcccc}
\toprule
Benchmark & Weak time (s) & Strong time (s) & Weak rows $K$ & DE budget \\
\midrule
FADE              & 7.4 & 5.0 & $\sim2640$ & $7\times24$ \\
Frac. RD (space)  & 1.8 & 3.2 & $\sim1440$ & $7\times24$ \\
Frac. Burgers     & 8.6 & 5.5 & $\sim2640$ & $7\times24$ \\
\botrule
\end{tabular*}
\end{table}

\subsection{Ablation: weak versus strong-form library, and fixed dictionaries}\label{subsec:ablation}

Table~\ref{tab:ablation} separates the three components of Weak-Pareto: weak measurements, continuous-order subset search, and exact-order polishing. The library-only contrast is decisive. With the same continuous-order selector and no polishing, Strong-Pareto achieves no complete FADE operator recovery at $10\%$ noise, whereas Weak-Pareto recovers all five. With the optimiser held fixed, the $0/5$-to-$5/5$ change isolates the weak library as the source of the robustness gain.

A fixed order dictionary is the natural alternative to continuous-order search, but it faces a resolution--conditioning trade-off. A coarse grid introduces order-discretisation error; a fine grid creates many nearly duplicate columns. Table~\ref{tab:conditioning} quantifies this effect for FADE. At $\Delta\beta=0.25$, the normalised weak dictionary already has condition number $2.6\times10^9$; at $\Delta\beta=0.10$, its mutual coherence is $0.987$ and its condition number reaches $8.4\times10^{15}$. Such a large condition number indicates severe numerical ill-conditioning, or near rank deficiency: small perturbations in the data or arithmetic can cause large changes in fitted coefficients. It is therefore a numerical indicator of regression instability. Weak-Pareto avoids this dense design by evaluating at most $c\le c_{\max}$ proposed columns at a time. The final direct evaluation removes interpolation error conditional on the selected model, while support selection can still depend on the order grid and search trajectory.

The empirical ablation confirms the theoretical distinction. Weak Grid-STRidge, which applies sequential threshold ridge regression (STRidge) on a fixed weak dictionary, does not recover the correct FADE support in any of the five runs even though its fixed grid contains nodes within $0.01$ of both true orders. Its small residual therefore reflects a coherent overcomplete dictionary that fits the weak equations without identifying the correct structure. Continuous-order Weak-Pareto recovers all five supports and operators. Disabling exact-order polishing leaves recovery unchanged and changes $e_\beta^{\max}$ only from $0.08$ to $0.07$; polishing is therefore a small refinement on this benchmark. Together, the ablations show that weak measurements provide noise robustness and continuous-order best-subset search provides reliable support selection.

\begin{table}[t]
\centering
\caption{Component ablation on FADE at $10\%$ noise, toggling the weak library, continuous-order search, and exact-order polishing.}\label{tab:ablation}
\footnotesize
\setlength{\tabcolsep}{0.4pt}
\begin{tabular*}{\textwidth}{@{\extracolsep\fill}lcccccccc}
\toprule
Method & Weak & Cont. & Polish & \shortstack{Supp./\\pow.} & \shortstack{Oper./\\struct.} & $e_\beta^{\max}$ & $e_\xi^{\max}$ & $\mathcal E_{\mathrm{fit}}$ \\
\midrule
\shortstack[l]{Strong-\\Pareto} & no & yes & no & 0/5 & 0/5 & -- & -- & $0.842\pm0.002$ \\
\shortstack[l]{Weak Grid-\\STRidge} & yes & no & no & 0/5 & 0/5 & -- & -- & $0.014\pm0.001$ \\
\shortstack[l]{Weak-Pareto\\(no polish)} & yes & yes & no & 5/5 & 5/5 & $0.08\pm0.03$ & $0.13\pm0.06$ & $0.022\pm0.001$ \\
\shortstack[l]{Full\\Weak-Pareto} & yes & yes & yes & 5/5 & 5/5 & $0.07\pm0.03$ & $0.12\pm0.09$ & $0.022\pm0.001$ \\
\botrule
\end{tabular*}
\end{table}

\begin{table}[t]
\caption{Conditioning of a fixed dense fractional dictionary as the order spacing $\Delta\beta$ shrinks: number of columns, mutual coherence, and condition number of the column-normalised weak library (noiseless FADE, orders over $[0.1,3]$).}\label{tab:conditioning}
\begin{tabular*}{\textwidth}{@{\extracolsep\fill}lccc}
\toprule
Order spacing $\Delta\beta$ & Columns & Mutual coherence & Condition number \\
\midrule
0.50 & 6  & 0.722 & $3.2{\times}10^{2}$ \\
0.25 & 12 & 0.921 & $2.6{\times}10^{9}$ \\
0.10 & 30 & 0.987 & $8.4{\times}10^{15}$ \\
0.05 & 59 & 0.997 & $1.3{\times}10^{16}$ \\
\botrule
\end{tabular*}
\end{table}

\subsection{Limitations and future work}\label{subsec:limits}

Two empirical limits are prominent. In the Riesz reaction--diffusion experiments, support and powers are often retained while nearby spatial orders and small reaction coefficients remain difficult to distinguish. Increasing the optimisation budget does not resolve this ambiguity, and Section~\ref{subsec:local-identifiability} shows that these spatial orders have substantially flatter coefficient-profiled directions than the directional FADE and Burgers terms. Their accurate clean profiles indicate sensitivity to perturbations once noise is present. In the fixed-support superunit diagnostic, the correct branch is retained at $1\%$ noise, but four of five temporal-order estimates reach the upper search bound. Remark~\ref{rem:temporal-target} links this behaviour to the endpoint-concentrated composed L1 target; alternative treatments of the initial rate are left for future work.

The theoretical guarantees are strongest for linear right-hand-side features. Nonlinear weak features average the corresponding strong features but reuse the noisy field and can therefore be biased. The support-conditioned sensitivity result is local; global identifiability of the joint discrete--continuous model class remains open. The differential-evolution search and local polishing are heuristic optimisation procedures, and rigorous global-convergence analysis is left for future work.

Several empirical directions remain open. Test-window scale is problem dependent: in the one-dimensional $K$-sweep, increasing the number of rows narrows the localised Gaussian test-function windows and eventually degrades operator recovery, while substantially broader Gaussian windows cause support under-selection in the two-dimensional diagnostic. The challenging Riesz cases also depend on the admitted powers and selection rule. The two-dimensional study establishes directional extensibility for dense, periodic, uniformly sampled data with $48{,}000$ overlapping weak rows. Sparse observations, nonlinear multidimensional discovery, and matched two-dimensional weak--strong comparisons remain future directions. Observation horizon and spectral content were not varied systematically, and overlapping test functions leave the disjoint validation rows statistically correlated.

Accordingly, we report support recovery, operator recovery, order error, coefficient error, and fit residual separately. Operator-specific care also remains essential: finite-domain and periodic fractional derivatives are different models, and each weak feature must use the adjoint of the operator it represents.

\section{Conclusion}\label{sec:conclusion}

Weak-Pareto combines two methodological advances for fractional equation discovery: an adjoint-consistent weak library and a continuous-order Pareto search. For linear right-hand-side terms, the weak formulation removes pointwise fractional differentiation from the measured field, and the variance analysis explains why this improves robustness as the grid is refined. For nonlinear terms, it provides integrated projections that reduce noise through averaging but may remain biased. The continuous-order encoding avoids the discretisation error and severe collinearity associated with fixed dictionaries, while the elbow search makes the trade-off between fit and complexity explicit. A support-conditioned local sensitivity analysis additionally quantifies when spatial-order changes can be absorbed by coefficient refitting, linking weak regression geometry to the observed hierarchy of noisy spatial-order errors.

The experiments support these advantages. On FADE and fractional Burgers, Weak-Pareto recovers the correct support in every seed at all tested multiplicative-noise levels up to $20\%$. The matched strong-form framework largely fails once noise is introduced, and the conclusion is unchanged under additive Gaussian noise. Component ablations show that the weak library drives noise robustness and that continuous-order search resolves the support-selection failure of the fixed-grid baseline on FADE. The superunit diagnostic gives 5/5 operator recovery through $0.5\%$ noise but 1/5 at $1\%$, marking an endpoint-sensitive limit. On the advection--diffusion benchmark, Weak-Pareto also yields more consistent operator recovery and lower wall-clock runtime than the adapted neural fractional-discovery framework under the same CPU environment, although the operator realisations are not identical. Together, the Riesz and superunit results show that support or branch selection can remain stable while fractional orders become weakly identifiable. The two-dimensional example further shows that the candidate tuple can absorb coordinate direction as an additional discrete label: both anisotropic benchmarks are recovered in every run through $20\%$ noise. Finally, the frozen-soil example shows that the weak-form representation can fit a fractional Kelvin model to irregular experimental data, with close agreement for silt and greater uncertainty for clay. Future work should address uncertainty-aware selection near identifiability limits, sparse and nonlinear multidimensional discovery, and application to experimental anomalous-transport systems.

\begin{appendices}

\section{Adjoint identities}\label{app:adjoints}

\paragraph{Riemann--Liouville integration by parts.} For $0<\gamma<1$ and sufficiently regular $f,\phi$, the left RL derivative~\eqref{eq:rl-def} satisfies
\begin{equation}
\int_a^b ({}_aD_z^\gamma f)(z)\,\phi(z)\dd z=\int_a^b f(z)\,({}_zD_b^\gamma\phi)(z)\dd z,
\label{eq:rl-ibp}
\end{equation}
i.e.\ a left derivative on $f$ becomes a right derivative on $\phi$. For $0<\gamma<1$ it suffices in the present application that $f$ be bounded and $\phi$ vanish at the right endpoint $b$. More generally, if $f\in L^q(a,b)$ with $q>1/(1-\gamma)$, then $({}_aI_z^{1-\gamma}f)(z)=O((z-a)^{1-\gamma-1/q})\to0$ as $z\to a^+$; bounded $f$ is a sufficient special case and applies here to $f=u-u(0,\cdot)$. With $\phi(b)=0$, one has $\frac{\dd}{\dd z}\bigl({}_zI_b^{1-\gamma}\phi\bigr)={}_zI_b^{1-\gamma}(\phi')$, from which~\eqref{eq:rl-ibp} follows by ordinary integration by parts and Fubini.

For higher orders $n-1<\gamma<n$, the analogous identity also requires the relevant left-endpoint traces of ${}_aI_z^{n-\gamma}f$ to vanish and $\phi^{(m)}(b)=0$ for $m=0,\dots,n-1$. In the present superunit Caputo application, $f=u-P_{1,a}u$ satisfies $f(a)=f'(a)=0$; for example, $u(\cdot,x)\in C^1[a,b]$ with locally H\"older-continuous $u_t$ is a sufficient regularity condition for these left traces to vanish.

If an endpoint condition fails, its continuum boundary term must be retained. The reported Gaussian tests instead use the exact discrete-adjoint construction of Section~\ref{subsec:discrete}, which preserves the implemented discrete inner-product identity without assuming vanishing Gaussian traces.

\paragraph{Caputo target.} Writing ${}_0^{C}\!D_t^\alpha u={}_0D_t^\alpha[u-P_{n-1,0}u]$ and applying the higher-order counterpart of Eq.~\eqref{eq:rl-ibp} gives Eq.~\eqref{eq:caputo-weak} for $0<\alpha<2$, $\alpha\ne1$, provided the test function satisfies the corresponding terminal conditions. For $0<\alpha<1$, the subtracted polynomial is $u(0,\cdot)$; for $1<\alpha<2$, it is $u(0,\cdot)+t\,\partial_tu(0,\cdot)$. The discrete target implements the same branch-specific correction through the transpose of the corresponding Caputo matrix.

\paragraph{Spectral operators.} For periodic fields, Parseval's identity and the (conjugate) symmetry of the multipliers in~\eqref{eq:spectral-def} give self-adjointness of the Riesz operator and the conjugate-multiplier adjoint of the directional operator stated in Section~\ref{subsec:adjoints}.

\section{Discrete operators and numerical verification}\label{app:discrete}

The discrete adjoint identity~\eqref{eq:discrete-adjoint} is verified for every operator family: matrix transposes for Gr\"unwald--Letnikov and one-sided finite-domain stencils, the transposed L1 matrix for Caputo time derivatives, and conjugate Fourier multipliers for periodic Riesz and directional operators. Random smooth-field tests satisfy $\inner{Af}{\phi}_h=\inner{f}{A^{\ast,h}\phi}_h$ to machine precision. The optional fractional-integral adjoint is verified in the same way.

We also checked the grid-refinement assumptions of Proposition~\ref{prop:variance} empirically. For $n_t=n_x=n\in\{24,32,48,64,96\}$, a fixed smooth periodic separable test function, a spectral order $\beta=1.5$, and $2000$ independent standard-Gaussian noise fields per grid, least-squares fits of log variance against $\log n$ gave slopes $-2.01$ for the weak feature and $2.97$ for the pointwise feature, close to the predicted $-2$ and $2\beta=3$. An independent deterministic calculation reproduces this check. This verification holds the test function on $Q$ fixed, exactly as assumed in the proposition; the separate $K$-sensitivity study of Appendix~\ref{app:ksens} examines what happens when the number and width of the localised Gaussian test functions are changed.

The exact adjoint check verifies the discrete transpose identities. Continuum consistency is assessed independently using analytic operator--function pairs. For $\mathcal R_\beta\sin(mx)=-|m|^\beta\sin(mx)$ with $\beta=1.7$ and $m=3$, the periodic FFT implementation has relative errors between $8\times10^{-15}$ and $2\times10^{-13}$ on grids of $32$--$256$ points. For ${}_0^CD_t^{0.7}t^3=\Gamma(4)t^{2.3}/\Gamma(3.3)$, the Caputo--L1 relative error decreases from $5.2\times10^{-3}$ to $3.6\times10^{-4}$ over $65$--$513$ points, with observed rates $1.28$--$1.29$, close to the expected $2-\alpha=1.3$. For the separately implemented superunit composition, ${}_0^CD_t^{1.3}t^3=\Gamma(4)t^{1.7}/\Gamma(2.7)$, the relative error over the non-initial rows decreases from $1.20\times10^{-2}$ to $9.99\times10^{-4}$ over the same grids, with observed rates $1.19$--$1.20$. We claim empirical convergence for this composed discretisation; a formal convergence order for the complete composition is not derived here. This verifies convergence of the active superunit code path independently of model selection. The semi-analytic fixed-support recovery experiment in Section~\ref{subsec:superunit} provides the complementary branch/order test on clean and noisy fields.

The nonlinear-bias formulas of Remark~\ref{rem:nonlinear-bias} were also checked numerically. For the periodic first derivative, the predicted leading bias is zero. The observed Monte Carlo means remain close to zero and, unlike the order-$1.7$ results, show no systematic growth with grid refinement. For a directional derivative of order $1.7$, the additive-Gaussian predicted biases at $n_x=64,128,256$ are $-8.39$, $-27.25$, and $-88.52$, while the observed means are $-8.36$, $-27.23$, and $-88.60$. Under multiplicative-uniform noise the corresponding predictions are $-1.85$, $-6.02$, and $-19.55$, versus observed means $-1.85$, $-6.03$, and $-19.58$. The successive additive-bias ratios are about $3.25$, matching the predicted grid-doubling factor $2^{1.7}$. The observed growth agrees with the predicted nonlinear-feature bias, confirming an intrinsic limitation of these data-weighted nonlinear features.

Finally, an oracle diagnostic profiles the held-out objective on the correct two-term support and operator modes while minimising over the positive Riesz order. Its five arms use clean data, fully noisy data, the noisy field with only the clean initial slice restored, a noisy target with a clean library, and a clean target with a noisy library. For the space-fractional benchmark the branch-aware clean profile selects the exact integer mode $\alpha=1$; at $2\%$ noise the fully noisy and target-only arms minimise at $0.820$ and $0.817$, whereas the library-only arm remains at the exact integer mode. Restoring only the initial slice gives $0.800$, the declared lower bound; the constrained profile therefore remains shifted and does not provide an interior minimum. For the time--space benchmark the clean and library-only minima are $0.820$, while the fully noisy and target-only minima are both $0.772$; restoring the initial slice gives $0.763$. This fixed-support diagnostic identifies the noisy temporal target as the dominant source of the observed shift in these examples. The direction of that shift is profile-specific and need not persist for other problems or scoring conventions.

\section{Experimental settings}\label{app:settings}

The one-dimensional benchmark grids are listed in Table~\ref{tab:benchmarks}; all spatial grids are periodic, and their sizes range from $80\times80$ to $150\times120$ in time--space samples. The two-dimensional example uses a $90\times80\times80$ time--space--space grid. The reported weak library uses tensor-product Gaussian test functions for most searches and Fourier spatial modes for the periodic high-order Riesz cases. The main configuration uses powers $p\in\{0,1,2\}$, $c_{\max}=4$, automatic plateau and selection-stability stopping, a relative-improvement tolerance $\delta=0.03$, and the signed elbow rule with two-point margin $0.15$. Appendix~\ref{app:ksens} examines sensitivity to the number of weak rows.

Fractional orders are represented on uniform, branch-confined grids: 47 nominal temporal nodes and 59 spatial nodes over the declared ranges in Table~\ref{tab:searchranges}. The separator $\epsilon_\alpha=10^{-3}$ prevents interpolation across $\alpha=1$, which is represented by a distinct exact-integer candidate. No noninteger true order is inserted into the grids or optimiser initialisation. The identity is also a distinct spatial candidate at $\beta=0$. After selection, direct operator evaluation at the polished orders removes interpolation error from the conditional refit, but the selected model can still depend on grid resolution and the search trajectory.

Differential evolution uses a SciPy population multiplier of $7$ and 24 generations. The ridge parameter is $\lambda=10^{-3}$ on column-normalised designs, the validation fraction is $0.25$, and pruning uses $\tau_{\mathrm{contrib}}=10^{-4}$ and $\tau_{\mathrm{abs}}=10^{-10}$, with a relative-coefficient fallback of $10^{-3}$ when library columns are unavailable. The residual guards are $10^{-14}$ for validation scoring and $10^{-12}$ for coefficient errors. Temporal boundary trimming is a control for pointwise feature construction only; it is not applied to the weak framework, whose discrete adjoint target retains the endpoint structure of the implemented operator. For the nonnegative FADE and integer advection--diffusion (ADE) fields, incorrect nonlinear candidate terms use $(u_+)^p$ to avoid amplifying noise-induced sign changes; the sign-changing Burgers field uses $u^p$. The challenging cases in Appendix~\ref{app:challenging} use the stated case-specific power set and, for the two-term Riesz case, the heuristic Akaike information criterion (AIC)-type selector.

\begin{table}[t]
\centering
\footnotesize
\setlength{\tabcolsep}{5pt}
\caption{Declared search domains per benchmark: the temporal-order interval $[\alpha_{\min},\alpha_{\max}]$, the spatial-order interval $[\beta_{\min},\beta_{\max}]$, the support-size cap $c_{\max}$, and the admitted power set $\mathcal A$. These are the differential-evolution bounds (Section~\ref{sec:pareto}); they encode coarse prior knowledge of the operator regime. No noninteger benchmark-true fractional order is inserted, whereas exact integer modes are included independently of the benchmark truth. If a temporal interval crosses one, its fractional portions and the exact integer candidate are searched separately. The identity operator is a discrete candidate at $\beta=0$; derivative orders are searched on the strictly positive part of the interval.}\label{tab:searchranges}
\begin{tabular*}{\textwidth}{@{\extracolsep{\fill}}lcccc}
\toprule
Benchmark & $[\alpha_{\min},\alpha_{\max}]$ & $[\beta_{\min},\beta_{\max}]$ & $c_{\max}$ & $\mathcal A$ \\
\midrule
FADE                    & $[0.60,1.05]$ & $[0.50,2.00]$ & 4 & $\{0,1,2\}$ \\
Frac.\ RD (space)       & $[0.80,1.15]$ & $[0.00,2.10]$ & 4 & $\{0,1,2\}$ \\
Frac.\ RD (time--space) & $[0.65,1.00]$ & $[0.00,1.90]$ & 4 & $\{0,1,2\}$ \\
Frac.\ Burgers          & $[0.85,1.15]$ & $[0.50,2.00]$ & 4 & $\{0,1,2\}$ \\
2-D directional A/B      & $[0.55,1.25]$ & $[0.50,2.50]$ & 4 & $\{0\}$ \\
Superunit diagnostic    & $[0.65,1.85]$ & $[0.50,2.50]$ & 1 & $\{0\}$ \\
ADE (challenging)       & $[0.80,1.20]$ & $[0.70,2.00]$ & 4 & $\{0\}$ \\
Two-term Riesz (chall.) & $[0.80,1.15]$ & $[0.30,3.10]$ & 4 & $\{0\}$ \\
\botrule
\end{tabular*}
\end{table}

For ADE, the classical second derivative $\beta=2$ is simply the upper integer endpoint of the search range. A supplementary reduced-sampling diagnostic retains every second temporal snapshot without interpolation or imputation and is included only as an implementation check.

\paragraph{Local spatial-order diagnostic.}
Eq.~\eqref{eq:weak-order-sensitivity} is evaluated with a centred finite difference of step $10^{-4}$ on exact weak features. The fixed-support profiles use 61 equally spaced trial orders over the benchmark's positive spatial-order search interval, augmented by the true order, followed by bounded one-dimensional refinement around the best grid point with tolerance $10^{-5}$. The true temporal branch/order, support, and all other spatial orders are fixed; coefficients, training/validation rows, ridge parameter, and validation score follow the main protocol. The reproduction script and all $0\%$, $2\%$, $5\%$, and $10\%$ per-seed profiles are archived in Online Resource 1.

\subsection{Implementation of the best-subset selection}\label{subsec:objective}
Eqs.~\eqref{eq:best-subset}--\eqref{eq:ridge} define the outer objective and inner coefficient fit. Weak rows are split deterministically into training and validation subsets, with validation fraction $0.25$. Training columns are normalised before ridge regression and the fitted coefficients are mapped back to the original scale. The principal criterion for differential evolution, Pareto dominance, elbow selection, and stopping is the variance-normalised validation score in Eq.~\eqref{eq:val}. Training error, validation error, the selection objective, and the post-refit full-data residual remain distinct fields: the objective is normally $\log_{10}$ of the normalised validation mean-squared error (MSE) plus any declared search penalty, whereas $\mathcal E_{\mathrm{fit}}$ is computed only after exact-order refitting. The duplicate-order penalty is
\begin{equation}
\Pi_{\mathrm{dup}}=\lambda_{\mathrm{dup}}\!\sum_{\substack{i<j\\p_i=p_j}}\!\left(1-\frac{|\beta_i-\beta_j|}{\delta_\beta}\right)_{+},
\qquad \lambda_{\mathrm{dup}}=0.02,\quad\delta_\beta=0.04.
\label{eq:duplicate-penalty}
\end{equation}
The branch separator is $\epsilon_\alpha=10^{-3}$. Differential evolution is run independently on each nonempty temporal mode, using the branch-specific dimensions and bounds stated in Section~\ref{subsec:bestsubset}; its best candidates are then compared by the same objective. The exact integer mode evaluates $\partial_t$ directly and never interpolates neighbouring fractional features. The selected mode is retained during inactive-term pruning, exact-order polishing, and the final full-row coefficient refit.

\section{Nonlinear fractional Burgers solver}\label{app:burgers}

The nonlinear benchmark $\partial_tu=-u\,\partial_xu+\nu D_x^\beta u$ is integrated pseudospectrally on $480$ periodic spatial points over $[0,30)$, using fourth-order Runge--Kutta with internal step $\Delta t_{\mathrm{fine}}=0.004$ up to $T=12$. The quadratic flux $-\tfrac12\partial_x(u^2)$ is dealiased by the $2/3$ rule. Retaining every fourth spatial point and $150$ endpoint-excluded temporal snapshots gives the reported $150\times120$ grid. The conjugate-symmetric multiplier keeps the field real, while $\operatorname{Re}(\mathrm i\kappa)^\beta<0$ for $1<\beta<2$ supplies dissipation. The parameters $\nu=0.25$ and $\beta=1.7$ keep the solution smooth while maintaining comparable nonlinear and diffusive contributions.

\section{Additional experiments: challenging cases and hyperparameter sensitivity}\label{app:challenging}

We report two cases that probe the limits of the method under the non-restrictive main-text settings (plateau stopping on, $c_{\max}=4$, powers $p\in\{0,1,2\}$, elbow selection). They are at opposite ends of the model class:
\begin{itemize}
\item \textbf{No fractional derivative (ADE).} The integer-order advection--diffusion equation $\partial_t u=-\partial_x u+0.25\,\partial_x^2 u$, with true terms $(0,1,-1)$ and $(0,2,0.25)$ (directional operator).
\item \textbf{More than one fractional spatial derivative (two-term Riesz).} The equation $\partial_t u=0.05\,\Rop_{0.55}u+0.005\,\Rop_{2.8}u$, with true terms $(0,0.55,0.05)$ and $(0,2.8,0.005)$.
\end{itemize}

Under the common main-text settings, both challenging cases fail structurally (Table~\ref{tab:appendix}). For ADE, the search selects the correct support size but replaces $\partial_x^2u$ with a nonlinear candidate that is absent from the true equation. For the two-term Riesz equation, the dominant high-order term is identified, but the low-order term contributes less than the noise floor at $10\%$; consequently, the elbow selects only one term. Because the support is then incomplete, order and coefficient errors are not reported for the default rows.

Case-specific restrictions improve support recovery but not complete operator identification. Limiting ADE to $p=0$ recovers the support and powers in $4/5$ seeds. For the two-term Riesz case, combining $p=0$ with a heuristic AIC-type selector recovers the two-term support in all five seeds. This selector uses $n_{\mathrm{val}}\log\mathrm{MSE}_{\mathrm{val}}+2k$ on held-out rows. Because changing $\alpha$ changes the response, this score is used only as a case-specific heuristic; the variance-normalised validation score remains the principal selector. Neither adjustment achieves complete operator recovery at $10\%$ noise, confirming sensitivity to the admitted powers, selection rule, and order identifiability.

\begin{table}[t]
\caption{Challenging cases under non-restrictive defaults versus case-specific adjustments at $10\%$ noise. Errors are conditioned on support/power recovery; $\mathcal E_{\mathrm{fit}}$ is over all seeds.}\label{tab:appendix}
\footnotesize
\setlength{\tabcolsep}{0.2pt}
\begin{tabular*}{\textwidth}{@{\extracolsep\fill}llccccc}
\toprule
Benchmark & Setting & \shortstack{Support/\\power} & \shortstack{Operator\\structure} & $e_\beta^{\max}$ & $e_\xi^{\max}$ & $\mathcal E_{\mathrm{fit}}$ \\
\midrule
\multirow{2}{*}{ADE} & default & 0/5 & 0/5 & -- & -- & $0.031\pm0.008$ \\
 & $p\in\{0\}$ & 4/5 & 0/5 & $0.17\pm0.21$ & $0.46\pm0.19$ & $0.037\pm0.004$ \\
\multirow{2}{*}{\shortstack[l]{Two-term\\Riesz}} & default & 0/5 & 0/5 & -- & -- & $0.238\pm0.014$ \\
 & $p\in\{0\}$, AIC-type & 5/5 & 0/5 & $0.21\pm0.14$ & $0.38\pm0.11$ & $0.236\pm0.013$ \\
\botrule
\end{tabular*}
\end{table}

\subsection{Two-point elbow-margin sensitivity}\label{subsec:elbow-margin-sensitivity}

The special two-point rule in Section~\ref{subsec:sweep} uses the margin $m_2=0.15$ only when the available front contains $c=1$ and $c=2$. Define
\[
\Delta_{12}=\log_{10}\mathcal E_1^\star-\log_{10}\mathcal E_2^\star
=\log_{10}(\mathcal E_1^\star/\mathcal E_2^\star),
\]
so that $c=2$ clears the two-point rule when $\Delta_{12}>m_2$. Table~\ref{tab:elbow-margin-sensitivity} applies $m_2\in\{0.10,0.15,0.20\}$ to the same paper-budget $c=1,2$ fronts at $10\%$ noise, so the comparison isolates the margin itself. The $c=2$ candidates are selection-stage fits from these restricted two-point fronts. Section~\ref{subsec:mainresults} reports the fully searched and refined models, so the corresponding parameter estimates can differ.

\begin{table}[t]
\centering
\caption{Sensitivity of the two-point elbow decision at $10\%$ multiplicative noise (five seeds). The interval gives the observed range of $\Delta_{12}$ across seeds; the final three columns count seeds in which $c=2$ clears the stated margin. The reported default is $m_2=0.15$.}\label{tab:elbow-margin-sensitivity}
\footnotesize
\setlength{\tabcolsep}{4pt}
\begin{tabular*}{\textwidth}{@{\extracolsep\fill}lcccc}
\toprule
Benchmark & Range of $\Delta_{12}$ & $m_2=0.10$ & $m_2=0.15$ & $m_2=0.20$ \\
\midrule
FADE & $0.971$--$1.124$ & $5/5$ & $5/5$ & $5/5$ \\
Frac.\ Burgers & $2.019$--$2.269$ & $5/5$ & $5/5$ & $5/5$ \\
Frac.\ RD (space) & $0.107$--$0.267$ & $5/5$ & $3/5$ & $1/5$ \\
Frac.\ RD (time--space) & $0.384$--$0.663$ & $5/5$ & $5/5$ & $5/5$ \\
\botrule
\end{tabular*}
\end{table}

The FADE, fractional Burgers, and time--space reaction--diffusion two-point decisions are unchanged throughout this neighbourhood of the default margin. The space-fractional Riesz case is selector-sensitive at $10\%$ noise: full-selector reruns give support/power recovery of $5/5$, $3/5$, and $1/5$ at margins $0.10$, $0.15$, and $0.20$, respectively, while complete operator recovery remains $0/5$ throughout. The margin therefore affects support retention in this severe case, while the spatial-order difficulty persists. Online Resource~1 archives the complete $c=1,2$ audit and the full-selector verification.

\section{Sensitivity to the number of weak rows}\label{app:ksens}

The number of weak rows $K=n_t^{\mathrm{test}}n_x^{\mathrm{test}}$ controls measurement resolution, not the candidate class. Table~\ref{tab:ksens} varies $K$ seventeen-fold on FADE at $10\%$ noise. Support and powers are recovered in all five seeds throughout, and complete operator recovery remains $5/5$ up to the main setting $K=2640$. At $K=5270$, support recovery remains perfect but operator recovery falls to $1/5$, with larger spatial-order and coefficient errors.

At large $K$, the localised Gaussian test-function windows generated by the paper's count-to-width rule become narrower in time and space; narrower localisation broadens their spectra, allowing more high-wavenumber noise to enter each weak row. Moderate $K$ therefore improves coefficient precision and reduces construction cost. The few-column continuous-order design remains well conditioned, with condition number $O(1)$, unlike the dense fixed dictionary in Table~\ref{tab:conditioning}. The optimal test-function family and bandwidth remain problem dependent; a matched comparison of Gaussian, compact-bump, and Fourier families is left to future work.

\begin{table}[t]
\caption{Sensitivity to the number of weak rows $K=n_t^{\mathrm{test}}\times n_x^{\mathrm{test}}$ on FADE at $10\%$ noise (five seeds). Reported dispersions are sample standard deviations. Under the Gaussian count-to-width rule, the number and width of localised test functions are coupled; larger $K$ therefore means narrower Gaussian test-function windows. The $K=2640$ row is the main setting.}\label{tab:ksens}
\centering
\footnotesize
\setlength{\tabcolsep}{0pt}
\begin{tabular*}{\textwidth}{@{\extracolsep\fill}lcccccccc}
\toprule
 Test grid & $K$ & \shortstack{Supp./\\pow.} & \shortstack{Oper./\\struct.} & $e_\alpha$ & $e_\beta^{\max}$ & $e_\xi^{\max}$ & $\mathcal E_{\mathrm{fit}}$ & $\operatorname{cond}_2(\widetilde\Tmat)$ \\
\midrule
$15{\times}21$ & 315  & 5/5 & 5/5 & $0.002\pm0.002$ & $0.04\pm0.03$ & $4.7{\times}10^{-2}$ & $8.6{\times}10^{-3}$ & $1.60$ \\
$22{\times}30$ & 660  & 5/5 & 5/5 & $0.002\pm0.002$ & $0.04\pm0.03$ & $5.0{\times}10^{-2}$ & $9.9{\times}10^{-3}$ & $1.60$ \\
$31{\times}43$ & 1333 & 5/5 & 5/5 & $0.002\pm0.002$ & $0.03\pm0.03$ & $5.4{\times}10^{-2}$ & $1.4{\times}10^{-2}$ & $1.59$ \\
$44{\times}60$ & 2640 & 5/5 & 5/5 & $0.002\pm0.002$ & $0.07\pm0.03$ & $1.2{\times}10^{-1}$ & $2.2{\times}10^{-2}$ & $1.59$ \\
$62{\times}85$ & 5270 & 5/5 & 1/5 & $0.002\pm0.001$ & $0.22\pm0.07$ & $6.3{\times}10^{-1}$ & $3.4{\times}10^{-2}$ & $1.59$ \\
\botrule
\end{tabular*}
\end{table}

\section{Alternative noise law}\label{app:altnoise}

The main experiments use multiplicative uniform perturbations because they preserve the local signal scale. To test whether the weak-versus-strong gap depends on that particular law, Table~\ref{tab:gaussian-noise} repeats the $10\%$ FADE and fractional Burgers comparisons with independent additive Gaussian noise,
\[
\widetilde u=u+0.10\,\operatorname{std}(u)Z,\qquad Z_{ij}\sim\mathcal N(0,1),
\]
using the same five seeds, search ranges, and optimisation budgets as the main experiments. The same noisy field is supplied to both methods for every seed. Weak-Pareto recovers the correct support in all five runs for both equations and the complete operator in $4/5$ seeds for each; the strong-form framework does not recover the correct support in any run for either equation. The result is consistent with Corollary~\ref{cor:multiplicative}: the averaging advantage of the weak library persists under additive Gaussian noise.

\begin{table}[t]
\centering
\caption{Additive-Gaussian robustness with noise standard deviation equal to $10\%$ of the clean-field standard deviation (five seeds). Parameter errors are conditioned on support/power recovery; $\mathcal E_{\mathrm{fit}}$ is over all seeds.}\label{tab:gaussian-noise}
\footnotesize
\setlength{\tabcolsep}{0.5pt}
\begin{tabular*}{\textwidth}{@{\extracolsep\fill}llccccc}
\toprule
Benchmark & Method & Supp. & Op. & $e_\beta^{\max}$ & $e_\xi^{\max}$ & $\mathcal E_{\mathrm{fit}}$ \\
\midrule
\multirow{2}{*}{FADE} & Weak & 5/5 & 4/5 & $0.134\pm0.033$ & $0.245\pm0.123$ & $0.030\pm0.003$ \\
 & Strong & 0/5 & 0/5 & -- & -- & $0.838\pm0.001$ \\
\multirow{2}{*}{Frac. Burgers} & Weak & 5/5 & 4/5 & $0.004\pm0.003$ & $0.017\pm0.011$ & $0.080\pm0.004$ \\
 & Strong & 0/5 & 0/5 & -- & -- & $0.831\pm0.001$ \\
\botrule
\end{tabular*}
\end{table}

\section{Forward-model validation}\label{app:forward}

Weak residual and trajectory reproduction measure different properties. We therefore integrate each representative discovered FPDE from the benchmark initial condition and report
\begin{equation}
e_{\mathrm{field}}=\frac{\lVert u_{\mathrm{disc}}-u\rVert_2}{\lVert u\rVert_2+\eps}.
\label{eq:efield}
\end{equation}
The spatial integrator uses the benchmark's declared periodic Riesz or directional multiplier. Time integration uses the exact exponential propagator for linear integer-time equations, an adaptively substepped and dealiased Runge--Kutta scheme for Burgers, and the Caputo L1 scheme for fractional time. The same solver is also run with the true parameters to quantify numerical discrepancy (Table~\ref{tab:forward}). For time--space reaction--diffusion, the generator and evaluator share the L1 discretisation, so this row is a self-consistency check. FADE is generated semi-analytically; both discovered-model errors are no larger than the true-parameter solver discrepancy, making those rows inconclusive. All simulations start from the discovery initial condition and cover the same horizon, so the test measures trajectory reproduction. The ADE control is omitted because the generic periodic evaluator uses a different operator convention from that finite-domain dataset; even the true-parameter simulation exceeds the self-consistency threshold.

\begin{table}[t]
\caption{Forward-model validation: normalised field error $e_{\mathrm{field}}$ between the discovered-model simulation and the clean reference, with the true-parameter solver discrepancy for context. This is a same-initial-condition trajectory-reproduction test; held-out prediction is outside its scope.}\label{tab:forward}
\centering
\footnotesize
\setlength{\tabcolsep}{3pt}
\begin{tabular*}{\textwidth}{@{\extracolsep\fill}lccc}
\toprule
Benchmark & Noise & $e_{\mathrm{field}}$ & Solver discrepancy $e_{\mathrm{field}}^{\mathrm{true}}$ \\
\midrule
\multirow{2}{*}{FADE} & 0\% & $4.74{\times}10^{-3}$ & $1.01{\times}10^{-2}$ \\
 & 10\% & $9.93{\times}10^{-3}$ & $1.01{\times}10^{-2}$ \\
\multirow{2}{*}{Frac. RD (space)} & 0\% & $9.12{\times}10^{-5}$ & $1.74{\times}10^{-16}$ \\
 & 5\% & $2.73{\times}10^{-3}$ & $1.74{\times}10^{-16}$ \\
\multirow{2}{*}{Frac. RD (time--space)} & 0\% & $2.00{\times}10^{-4}$ & $1.81{\times}10^{-16}$ \\
 & 5\% & $3.15{\times}10^{-3}$ & $1.81{\times}10^{-16}$ \\
\multirow{2}{*}{Frac. Burgers} & 0\% & $6.34{\times}10^{-4}$ & $4.98{\times}10^{-6}$ \\
 & 10\% & $2.05{\times}10^{-3}$ & $4.98{\times}10^{-6}$ \\
\botrule
\end{tabular*}
\end{table}

\section{Closed-form discovered equations}\label{app:equations}

Table~\ref{tab:equations} lists representative clean and noisy discoveries, including the selected powers, orders, and post-pruning coefficients. Each representative seed has the median worst spatial-order error, avoiding a best-case presentation; a dagger marks failure of the operator-structure criterion. The noisy reaction--diffusion representatives attain their lower temporal bounds ($\widehat{\alpha}=0.8000$ and $0.6500$); these boundary-attaining estimates are consistent with Appendix~\ref{app:discrete} and are not interior optima.

\begin{table}[t]
\caption{Representative discovered equations versus ground truth. The representative seed has the median $e_\beta^{\max}$; a dagger marks an unrecovered operator structure. Exact integer temporal modes are printed as $\partial_t$. Exact integer orders in the ground-truth equations use derivative shorthand, whereas continuously estimated spatial orders are printed numerically, including estimates that round to $1.00$.}\label{tab:equations}
\centering
\footnotesize
\setlength{\tabcolsep}{3pt}
\begin{tabular*}{\textwidth}{@{\extracolsep\fill}lll}
\toprule
Benchmark & Noise & Discovered equation (representative seed) \\
\midrule
\multicolumn{3}{l}{\emph{FADE} --- true: $D_t^{0.80}u=-1.00\,u_x+0.50\,D_x^{1.70}u$} \\
\multirow{2}{*}{FADE} & 0\% & $D_t^{0.7990}u=-0.99\,D_x^{1.00}u+0.49\,D_x^{1.73}u$ \\
 & 10\% & $D_t^{0.8038}u=-1.11\,D_x^{1.03}u+0.63\,D_x^{1.61}u$ \\
\midrule
\multicolumn{3}{l}{\emph{Frac. RD (space)} --- true: $\partial_tu=0.04\,u+0.18\,\mathcal R_{1.65}u$} \\
\multirow{2}{*}{\shortstack[l]{Frac. RD\\(space)}} & 0\% & $\partial_tu=0.04\,u+0.18\,\mathcal R_{1.65}u$ \\
 & 5\% & $D_t^{0.8000}u=-0.06\,\mathcal R_{0.27}u+0.20\,\mathcal R_{1.53}u$\textsuperscript{\(\dagger\)} \\
\midrule
\multicolumn{3}{l}{\emph{Frac. RD (time--space)} --- true: $D_t^{0.82}u=0.03\,u+0.12\,\mathcal R_{1.55}u$} \\
\multirow{2}{*}{\shortstack[l]{Frac. RD\\(time--space)}} & 0\% & $D_t^{0.8208}u=0.03\,u+0.12\,\mathcal R_{1.55}u$ \\
 & 5\% & $D_t^{0.6500}u=-0.06\,\mathcal R_{0.29}u+0.16\,\mathcal R_{1.44}u$\textsuperscript{\(\dagger\)} \\
\midrule
\multicolumn{3}{l}{\emph{Frac. Burgers} --- true: $\partial_tu=0.25\,D_x^{1.70}u-1.00\,u\,u_x$} \\
\multirow{2}{*}{Frac. Burgers} & 0\% & $\partial_tu=0.25\,D_x^{1.70}u-1.00\,u\,D_x^{1.00}u$ \\
 & 10\% & $\partial_tu=0.25\,D_x^{1.70}u-1.00\,u\,D_x^{1.00}u$ \\
\botrule
\end{tabular*}
\end{table}

\section{Two-dimensional settings and sensitivity}\label{app:twod}

The two-dimensional results use an example extension distributed in Online Resource 1. It reuses the one-dimensional implementations of the Gaussian test basis and the L1 Caputo adjoint, while its data generator includes a complex-capable evaluator of $E_{\alpha,1}$ because directional Fourier multipliers are complex. The generator and discovery code are independent in time: data use semi-analytic Mittag--Leffler propagation, whereas discovery evaluates the Caputo target through the transposed L1 matrix. The field contains five conjugate Fourier-mode pairs, and the supplied archive records the data hashes, software environment, per-seed estimates, validation curves, and complete summary.

The default spatial width applies the one-dimensional paper rule independently to $x$ and $y$. Table~\ref{tab:twod-width} compares three fixed fractions of the domain length on Benchmark~\eqref{eq:twod-b} at $5\%$ noise. The inherited rule and $0.10L$ recover the complete structure in every seed, whereas $0.16L$ and $0.24L$ select only two terms. This complements Appendix~\ref{app:ksens}, where narrower localised Gaussian test-function windows eventually degrade operator recovery: no two-dimensional tuning was needed here, but test-window scale remains problem dependent.

\begin{table}[t]
\centering
\caption{Spatial-window sensitivity for Benchmark~\eqref{eq:twod-b} at $5\%$ multiplicative noise (five seeds). Errors are conditioned on support and direction recovery.}\label{tab:twod-width}
\footnotesize
\setlength{\tabcolsep}{1.2pt}
\begin{tabular*}{\textwidth}{@{\extracolsep\fill}lccccc}
\toprule
Width rule & \shortstack{Support/\\direction} & Operator & $e_\alpha$ & $e_\beta^{\max}$ & $e_\xi^{\max}$ \\
\midrule
Inherited paper rule & $5/5$ & $5/5$ & $0.00177\pm0.00041$ & $0.00483\pm0.00045$ & $0.01336\pm0.00197$ \\
$0.10L$ & $5/5$ & $5/5$ & $0.00150\pm0.00034$ & $0.00875\pm0.00080$ & $0.00669\pm0.00119$ \\
$0.16L$ & $0/5$ & $0/5$ & -- & -- & -- \\
$0.24L$ & $0/5$ & $0/5$ & -- & -- & -- \\
\botrule
\end{tabular*}
\end{table}

Table~\ref{tab:twod-resolution} compares Benchmark~\eqref{eq:twod-a} at the reported and refined spatial grids. Both resolutions recover all five noise levels and seeds. The larger tensor-product set of weak rows improves the coefficient error at $20\%$ noise while leaving the order errors of comparable scale; the result serves as a grid-stability check. Sparse observation patterns would require a different row-construction and validation analysis and are outside the present scope.

\begin{table}[t]
\centering
\caption{Spatial-resolution check for Benchmark~\eqref{eq:twod-a}. Recovery counts aggregate 25 runs; errors are at $20\%$ noise.}\label{tab:twod-resolution}
\footnotesize
\setlength{\tabcolsep}{1.2pt}
\begin{tabular*}{\textwidth}{@{\extracolsep\fill}lccccc}
\toprule
\shortstack{Field\\grid} & \shortstack{Weak\\rows} & Operator & $e_\alpha$ & $e_\beta^{\max}$ & $e_\xi^{\max}$ \\
\midrule
$90\times80\times80$ & $48{,}000$ & $25/25$ & $0.00081\pm0.00054$ & $0.00217\pm0.00140$ & $0.0111\pm0.0023$ \\
$90\times112\times112$ & $94{,}080$ & $25/25$ & $0.00088\pm0.00034$ & $0.00270\pm0.00088$ & $0.00843\pm0.00271$ \\
\botrule
\end{tabular*}
\end{table}

\FloatBarrier

\section{Comparison with a contemporary neural fractional-discovery framework}\label{app:yu}
The Strong-Pareto versus Weak-Pareto (no polishing) ablation isolates the candidate-library effect under a common selector. We also compare Weak-Pareto with Yu et al.~\cite{yu2025fde} on the advection--diffusion benchmark. The full Yu et al. framework combines neural field reconstruction, automatic differentiation of integer derivatives, pointwise Gauss--Jacobi fractional derivatives, sparse regression, and global optimisation. An optimiser-only variant replaces the neural reconstruction with a deterministic quintic spline to isolate the downstream derivative and selection stages.

The comparison shares the FADE field, nominal equation, target orders, noise realisations, seeds, recovery tolerances, and CPU environment. It is not operator-identical: Weak-Pareto uses the periodic directional spectral operator that generated the field, whereas the Yu adaptation retains the one-sided finite-terminal Gauss--Jacobi approximation; the Riesz reaction--diffusion cases are outside its declared operator scope. The adapter fixes the training--validation split, estimates coefficients and the STRidge penalty from training rows only, and uses deterministic seeds. Its changes, budgets, and provenance controls are documented in Online Resource 1. Because the upstream snapshot is not redistributed, byte identity remains unverified. Runtime covers each complete fitting framework. Table~\ref{tab:yu} reports this controlled comparison.

\begin{table}[t]
\centering
\caption{Weak-Pareto versus the adapted neural fractional-discovery framework of Yu et al.~\cite{yu2025fde} on the advection--diffusion benchmark (five seeds). All rows use the same data, noise realisations, recovery tolerances, scoring convention, and CPU environment. Errors are mean $\pm$ standard deviation over runs with correct support and power; runtime is mean $\pm$ standard deviation over all five seeds. Table~\ref{tab:runtime} reports a separate uncached single-seed timing.}\label{tab:yu}
\footnotesize
\setlength{\tabcolsep}{1.0pt}
\begin{tabular*}{\textwidth}{@{\extracolsep\fill}lcccccc}
\toprule
Method & Noise (\%) & Op. rec. & $e_\alpha$ & $e_\beta^{\max}$ & $e_\xi^{\max}$ & Runtime (s) \\
\midrule
\multirow{3}{*}{Weak-Pareto} & 0 & 5/5 & $0.0010\pm0.0000$ & $0.026\pm0.000$ & $0.026\pm0.000$ & $6.6\pm0.3$ \\
 & 1 & 5/5 & $0.0007\pm0.0002$ & $0.025\pm0.005$ & $0.022\pm0.007$ & $6.7\pm0.3$ \\
 & 5 & 5/5 & $0.0008\pm0.0006$ & $0.020\pm0.016$ & $0.028\pm0.024$ & $6.7\pm0.2$ \\
\midrule
\multirow{3}{*}{Yu framework} & 0 & 1/5 & $0.012\pm0.002$ & $0.163\pm0.017$ & $0.097\pm0.010$ & $323\pm52$ \\
 & 1 & 2/5 & $0.011\pm0.001$ & $0.152\pm0.018$ & $0.090\pm0.009$ & $332\pm10$ \\
 & 5 & 4/5 & $0.004\pm0.003$ & $0.081\pm0.052$ & $0.044\pm0.027$ & $279\pm62$ \\
\midrule
\multirow{3}{*}{Yu optimiser-only} & 0 & 0/5 & $0.017\pm0.000$ & $0.204\pm0.003$ & $0.129\pm0.001$ & $25.4\pm4.3$ \\
 & 1 & 0/5 & -- & -- & -- & $10.8\pm0.6$ \\
 & 5 & 0/5 & -- & -- & -- & $11.6\pm3.0$ \\
\botrule
\end{tabular*}
\end{table}

Weak-Pareto recovers the complete FADE operator in all five seeds at $0\%$, $1\%$, and $5\%$ noise, with order and coefficient errors of a few percent. The adapted neural fractional-discovery framework recovers $1/5$, $2/5$, and $4/5$ operators; the non-monotone counts reflect run-to-run variability, not evidence that noise improves recovery. The optimiser-only variant never recovers the complete operator: at $0\%$ its spatial-order error is about $0.20$, and at positive noise it drops the fractional-diffusion term. On that CPU, Weak-Pareto takes $6.6$--$6.7$~s per run, versus $279$--$332$~s for the adapted framework and $11$--$25$~s for the optimiser-only variant. These are implementation-level timings for the tested configurations.

\end{appendices}

\section*{Acknowledgements}
The authors would like to thank Velmurugan Gandhi for his helpful discussions on fractional differential equations.

\section*{Supplementary material}
\noindent\textbf{Online Resource 1.} Source code, benchmark datasets, archived reference outputs, tutorials, tests, and scripts for reproducing the numerical results and figures reported in this article. We will also maintain the code and reproducibility materials at \url{https://github.com/Pongpisit-Thanasutives/Weak-Pareto}.

\bibliography{references}

\backmatter
\section*{Statements and Declarations}
\footnotesize
\setlength{\parskip}{1pt}
\noindent\textbf{Funding.} Pongpisit Thanasutives is supported by the research fund from the Special Postdoctoral Researcher (SPDR) program at RIKEN, Japan.\par
\noindent\textbf{Competing interests.} The authors have no relevant financial or non-financial interests to disclose.\par
\noindent\textbf{Author contributions.} Pongpisit Thanasutives designed the method and software, performed the investigation and validation, analysed the results, and wrote the manuscript. Yoshinobu Kawahara supervised the work and revised the manuscript. Both authors reviewed and approved the final manuscript.\par
\noindent\textbf{Use of generative AI.} During manuscript preparation, the authors used OpenAI ChatGPT and Anthropic Claude for language revision, manuscript drafting, and code review. The authors reviewed and verified all mathematical arguments, software implementations, results, and final text and take full responsibility for the work.\par
\noindent\textbf{Data and code availability.} Data and source code are provided in the Supplementary Material and will be maintained at \url{https://github.com/Pongpisit-Thanasutives/Weak-Pareto}. Third-party frozen-soil data are not redistributed; access instructions and provenance are documented in the cited paper.

\end{document}